%% file: concept-drift.tex
\documentclass[runningheads]{llncs}

\usepackage{bbm}
\usepackage{amsmath,amssymb}
\usepackage[mathscr]{eucal}

\input{tex-files/macros_varun.tex}
\newcommand{\change}{\Delta}

\newcommand{\changeseq}{\boldsymbol{\change}}
\newcommand{\C}{\mathbb C}

\newcommand{\X}{\mathcal X}
\renewcommand{\H}{\mathcal H}
\renewcommand{\P}{\mathbb P}
\newcommand{\Px}{\mathcal{P}}

\newcommand{\target}{h^{*}}
\newcommand{\targetseq}{\mathbf{h^{*}}}

\newcommand{\E}{\mathbb E}

\newcommand{\nats}{\mathbb{N}}
\newcommand{\reals}{\mathbb{R}}

\newcommand{\poly}{{\rm poly}}
\newcommand{\Log}{{\rm Log}}

\newcommand{\alg}{\mathcal A}

\newcommand{\DIS}{{\rm DIS}}
\newcommand{\Ball}{{\rm B}}
\newcommand{\er}{{\rm er}}

\newcommand{\ind}{\mathbbm{1}}

\newcommand{\vc}{d}

\newcommand{\sign}{{\rm sign}}
\newcommand{\dc}{\theta}

\newcommand{\argmin}{\mathop{\rm argmin}}

\newsavebox{\savepar}

\newenvironment{bigboxit}{\begin{center}\begin{lrbox}{\savepar}
\begin{minipage}[h]{4.6in}
\normalfont
\begin{flushleft}}
{\end{flushleft}\end{minipage}\end{lrbox}\fbox{\usebox{\savepar}}
\end{center}}

\newcommand{\citet}{\cite}
\newcommand{\citep}{\cite}
\newcommand{\citealp}{\cite}

\begin{document}
\mainmatter 

\title{Learning with a Drifting Target Concept}

\titlerunning{Learning with a Drifting Target Concept}

\author{Steve Hanneke \and Varun Kanade \and Liu Yang}

\authorrunning{Steve Hanneke, Varun Kanade, and Liu Yang}

\institute{Princeton, NJ USA.\\ 
\email{steve.hanneke@gmail.com}
\and
D\'{e}partement d'informatique, \'{E}cole normale sup\'{e}rieure, Paris, France.\\
\email{varun.kanade@ens.fr}
\and
IBM T.J. Watson Research Center, Yorktown Heights, NY USA.\\
\email{yangli@us.ibm.com}
}

\maketitle

\begin{abstract}
We study the problem of learning in the presence of a drifting target concept. Specifically,
we provide bounds on the error rate at a given time, given a learner with access to a history
of independent samples labeled according to a target concept that can change on each round. 
One of our main contributions is a refinement of the best previous results for 
polynomial-time algorithms for the space of linear separators under a uniform distribution.
We also provide general results for an algorithm capable of adapting to a variable rate of drift
of the target concept.
Some of the results also describe an active learning variant of this setting, and provide bounds on the
number of queries for the labels of points in the sequence sufficient to obtain the stated bounds
on the error rates.
\end{abstract}

\input{tex-files/definitions.tex}

\section{Background: $(\epsilon,S)$-Tracking Algorithms}
\label{sec:background}
\input{tex-files/pac-drift.tex}

\section{Adapting to Arbitrarily Varying Drift Rates}
\label{sec:general}
\input{tex-files/erm-passive.tex}

\section{Polynomial-Time Algorithms for Linear Separators}
\label{sec:halfspaces}
\input{tex-files/halfspaces.tex}

\section{General Results for Active Learning}
\label{sec:general-active}
\input{tex-files/general-active.tex}

\bibliographystyle{alpha}
\bibliography{bib-concept-drift}

\end{document}

%% file: tex-files/definitions.tex
\section{Introduction}

Much of the work on statistical learning has focused on 
learning settings in which the concept to be learned is static 
over time.
However, there are many application areas where this is not 
the case.  For instance, in the problem of face recognition, 
the concept to be learned actually changes over time as 
each individual's facial features evolve over time.  In this
work, we study the problem of learning with a drifting 
target concept.  Specifically, we consider a statistical 
learning setting, in which data arrive i.i.d. in a stream,
and for each data point, the learner is required to predict 
a label for the data point at that time.  We are then 
interested in obtaining low error rates for these predictions.
The target labels are generated from a function known to reside
in a given concept space, and at each time $t$ the target function
is allowed to change by at most some distance $\change_{t}$: that is, 
the probability the new target function disagrees with the previous
target function on a random sample is at most $\change_{t}$.

This framework has previously been studied in a number of articles.
The classic works of \citet{helmbold:91,helmbold:94,bartlett:96,long:99,bartlett:00} and \citet{barve:97}
together provide a general analysis of a
very-much related setting.  Though the objectives in these works are 
specified slightly differently, the results established there are 
easily translated into our present framework,
and we summarize many of the relevant results from this literature
in Section~\ref{sec:background}.

While the results in these classic works are general, the best guarantees 
on the error rates are only known for methods having no guarantees 
of computational efficiency.
In a more recent effort, the work of \citet{min_concept} studies this problem 
in the specific context of learning a homogeneous linear separator,
when all the $\change_{t}$ values are identical.
They propose a polynomial-time algorithm (based on the modified Perceptron 
algorithm of \citet{stream_perceptron}),
and prove a bound on the number of mistakes it makes as a function of
the number of samples, when the data distribution satisfies a
certain condition called ``$\lambda$-good'' (which generalizes a useful 
property of the uniform distribution on the origin-centered unit sphere).
However, their result is again worse than that obtainable by the known
computationally-inefficient methods.

Thus, the natural question is whether there exists a polynomial-time algorithm
achieving roughly the same guarantees on the error rates known for the inefficient methods.
In the present work, we resolve this question in the case of learning homogeneous
linear separators under the uniform distribution, by proposing a polynomial-time
algorithm that indeed achieves roughly the same bounds on the error rates
known for the inefficient methods in the literature.
This represents the main technical contribution of this work.

We also study the interesting problem of \emph{adaptivity} of an 
algorithm to the sequence of $\change_{t}$ values, in the setting where
$\change_{t}$ may itself vary over time.  Since the values $\change_{t}$
might typically not be accessible in practice, it seems important to 
have learning methods having no explicit dependence on the sequence $\change_{t}$.
We propose such a method below, and prove that it achieves roughly the 
same bounds on the error rates known for methods in the literature 
which require direct access to the $\change_{t}$ values.
Also in the context of variable $\change_{t}$ sequences, we discuss
conditions on the sequence $\change_{t}$ necessary and sufficient
for there to exist a learning method guaranteeing a \emph{sublinear}
rate of growth of the number of mistakes.

We additionally study an \emph{active learning} extension to this
framework, in which, at each time, after making its prediction,
the algorithm may decide whether or not to request access to the 
label assigned to the data point at that time.  In addition to guarantees on the 
error rates (for \emph{all} times, including those for which the label was not observed),
we are also interested in bounding the number of labels we expect the algorithm to
request, as a function of the number of samples encountered thus far.

\section{Definitions and Notation}
\label{sec:definitions}
Formally, in this setting, there is a fixed distribution $\Px$ over the instance space $\X$,
and there is a sequence of independent $\Px$-distributed unlabeled data $X_{1},X_{2},\ldots$.
There is also a concept space $\C$, and a sequence of target functions $\targetseq = \{\target_{1},\target_{2},\ldots\}$ in $\C$.
Each $t$ has an associated target label $Y_{t} = \target_{t}(X_{t})$.
In this context, a (passive) learning algorithm is required, on each round $t$, 
to produce a classifier $\hat{h}_{t}$ based on the observations $(X_{1},Y_{1}),\ldots,(X_{t-1},Y_{t-1})$,
and we denote by $\hat{Y}_{t} = \hat{h}_{t}(X_{t})$ the corresponding prediction by the algorithm
for the label of $X_{t}$.  For any classifier $h$, we define $\er_{t}(h) = \Px(x : h(x) \neq \target_{t}(x))$.
We also say the algorithm makes a ``mistake'' on instance $X_{t}$ if $\hat{Y}_{t} \neq Y_{t}$;
thus, $\er_{t}(\hat{h}_{t}) = \P( \hat{Y}_{t} \neq Y_{t} | (X_{1},Y_{1}),\ldots,(X_{t-1},Y_{t-1}) )$.

For notational convenience, we will suppose the $\target_{t}$ sequence is 
chosen independently from the $X_{t}$ sequence (i.e., $\target_{t}$ is chosen prior
to the ``draw'' of $X_{1},X_{2},\ldots \sim \Px$), and is not random. 

In each of our results, we will suppose $\targetseq$ is chosen from some set $S$ of 
sequences in $\C$.  In particular, we are interested in describing the sequence $\targetseq$
in terms of the magnitudes of \emph{changes} in $\target_{t}$ from one time to the next.
Specifically, for any sequence $\changeseq = \{\change_{t}\}_{t=2}^{\infty}$ in $[0,1]$, 
we denote by $S_{\changeseq}$ the set of all sequences $\targetseq$ in $\C$ such that,
$\forall t \in \nats$, $\Px(x : h_{t}(x) \neq h_{t+1}(x)) \leq \change_{t+1}$.

Throughout this article, we denote by $\vc$ the VC dimension of $\C$ \citep{vapnik:71},
and we suppose $\C$ is such that $1 \leq \vc < \infty$.
Also, for any $x \in \reals$, define $\Log(x) = \ln(\max\{x,e\})$.

%% file: tex-files/pac-drift.tex
As mentioned, the classic literature on learning with a drifting target concept
is expressed in terms of a slightly different model.  In order to relate those
results to our present setting, we first introduce the classic setting.
Specifically, we consider a model introduced by \citet{helmbold:94},
presented here in a more-general form inspired by \citet{bartlett:00}.
For a set $S$ of sequences $\{h_{t}\}_{t=1}^{\infty}$ in $\C$, 
and a value $\epsilon > 0$, an algorithm $\alg$ is said to be
\emph{$(\epsilon,S)$-tracking} if $\exists t_{\epsilon} \in \nats$ such that,
for any choice of $\targetseq \in S$, 
$\forall T \geq t_{\epsilon}$, 
the prediction $\hat{Y}_{T}$ produced by $\alg$ at time $T$ satisfies
\begin{equation*}
\P\left( \hat{Y}_{T} \neq Y_{T} \right) \leq \epsilon.
\end{equation*}
Note that the value of the probability in the above expression 
may be influenced by $\{X_{t}\}_{t=1}^{T}$, $\{\target_{t}\}_{t=1}^{T}$, 
and any internal randomness of the algorithm $\alg$.

The focus of the results expressed in this classical model is determining
sufficient conditions on the set $S$ for there to exist an $(\epsilon,S)$-tracking algorithm,
along with bounds on the sufficient size of $t_{\epsilon}$.
These conditions on $S$ typically take the form of an assumption on the 
drift rate, expressed in terms of $\epsilon$.  Below, we summarize 
several of the strongest known results for this setting. 

\subsection{Bounded Drift Rate}
\label{sec:classic-constant-drift}

The simplest, and perhaps most elegant, results for $(\epsilon,S)$-tracking algorithms
is for the set $S$ of sequences with a bounded drift rate.  Specifically, for any $\change \in [0,1]$, 
define $S_{\change} = S_{\changeseq}$, where $\changeseq$ is such that $\change_{t+1} = \change$ for every $t \in \nats$.
The study of this problem was initiated in the original work of \citet{helmbold:94}.  
The best known general results are due to \citet{long:99}: namely, 
that for some $\change_{\epsilon} = \Theta( \epsilon^{2} / \vc )$,
for every $\epsilon \in (0,1]$, there exists an $(\epsilon,S_{\change})$-tracking algorithm for all values 
of $\change \leq \change_{\epsilon}$.\footnote{In fact, \citet{long:99} also allowed the distribution 
$\Px$ to vary gradually over time.  For simplicity, we will only discuss the case of fixed $\Px$.}
This refined an earlier result of \citet{helmbold:94} by a logarithmic factor.
\citet{long:99} further argued that this result can be achieved with $t_{\epsilon} = \Theta(\vc/\epsilon)$.
The algorithm itself involves a beautiful modification of the one-inclusion graph prediction 
strategy of \citet{haussler:94}; since its specification is somewhat involved,
we refer the interested reader to the original work of \citet{long:99} for the details.

\subsection{Varying Drift Rate: Nonadaptive Algorithm}
\label{sec:classic-varying-drift}

In addition to the concrete bounds for the case $\targetseq \in S_{\change}$, 
\citet{helmbold:94} additionally present an elegant general result.  Specifically, 
they argue that, for any $\epsilon > 0$, and any $m = \Omega\left( \frac{\vc}{\epsilon}\Log\frac{1}{\epsilon} \right)$,
if $\sum_{i=1}^{m} \Px(x : \target_{i}(x) \neq \target_{m+1}(x)) \leq m \epsilon / 24$, then 
for $\hat{h} = \argmin_{h \in \C} \sum_{i=1}^{m} \ind[ h(X_{i}) \neq Y_{i} ]$, 
$\P( \hat{h}(X_{m+1}) \neq \target_{m+1}(X_{m+1}) ) \leq \epsilon$.\footnote{They in fact
prove a more general result, which also applies to methods approximately minimizing
the number of mistakes, but for simplicity we will only discuss this basic version of the result.}
This result immediately inspires an algorithm $\alg$ which, at every time $t$, 
chooses a value $m_{t} \leq t-1$, and predicts $\hat{Y}_{t} = \hat{h}_{t}(X_{t})$,
for $\hat{h}_{t} = \argmin_{h \in \C} \sum_{i=t-m_{t}}^{t-1} \ind[ h(X_{i}) \neq Y_{i} ]$.
We are then interested in choosing $m_{t}$ to minimize the value of $\epsilon$ obtainable
via the result of \citet{helmbold:94}.  However, that method is based on the 
values $\Px( x : \target_{i}(x) \neq \target_{t}(x) )$, which would typically not 
be accessible to the algorithm.  However, suppose instead we have access to a
sequence $\changeseq$ such that $\targetseq \in S_{\changeseq}$.  
In this case, we could approximate $\Px( x : \target_{i}(x) \neq \target_{t}(x) )$
by its \emph{upper bound} $\sum_{j = i+1}^{t} \change_{j}$.  In this case, 
we are interested choosing $m_{t}$ to minimize the smallest value of $\epsilon$ 
such that $\sum_{i=t-m_{t}}^{t-1} \sum_{j=i+1}^{t} \change_{j} \leq m_{t} \epsilon / 24$
and $m_{t} = \Omega\left( \frac{\vc}{\epsilon} \Log\frac{1}{\epsilon} \right)$.
One can easily verify that this minimum is obtained at a value
\begin{equation*}
m_{t} = \Theta\left( \argmin_{m \leq t-1} \frac{1}{m} \sum_{i=t-m}^{t-1} \sum_{j=i+1}^{t} \change_{j} + \frac{\vc \Log(m/\vc)}{m} \right),
\end{equation*}
and via the result of \citet{helmbold:94} (applied to the sequence $X_{t-m_{t}},\ldots,X_{t}$)
the resulting algorithm has
\begin{equation}
\label{eqn:hl94}
\P\left( \hat{Y}_{t} \neq Y_{t} \right) \leq O\left( \min_{1 \leq m \leq t-1} \frac{1}{m} \sum_{i=t-m}^{t-1} \sum_{j=i+1}^{t} \change_{j} + \frac{\vc \Log(m/\vc)}{m} \right).
\end{equation}

As a special case, if every $t$ has $\change_{t} = \change$ for a fixed value $\change \in [0,1]$, 
this result recovers the bound $\sqrt{ \vc \change \Log(1/\change) }$,
which is only slightly larger than that obtainable from the best bound of \citet{long:99}.
It also applies to far more general and more intersting sequences $\changeseq$,
including some that allow periodic large jumps (i.e., $\change_{t} = 1$ for some indices $t$), 
others where the sequence $\change_{t}$ converges to $0$, and so on.
Note, however, that the algorithm obtaining this bound
directly depends on the sequence $\changeseq$.
One of the contributions of the present work is to remove this requirement, while
maintaining essentially the same bound, though in a slightly different form.

\subsection{Computational Efficiency}
\label{sec:classic-consistency}

\citet{helmbold:94} also proposed a reduction-based approach, which 
sometimes yields computationally efficient methods, though the tolerable $\change$ 
value is smaller.  Specifically, given any (randomized) polynomial-time algorithm $\alg$
that produces a classifier $h \in \C$ with $\sum_{t=1}^{m} \ind[ h(x_{t}) \neq  y_{t} ] = 0$
for any sequence $(x_1,y_1),\ldots,(x_m,y_m)$ for which such a classifier $h$ exists
(called the \emph{consistency problem}),
they propose a polynomial-time algorithm that is $(\epsilon,S_{\change})$-tracking
for all values of $\change \leq \change_{\epsilon}^{\prime}$, 
where $\change_{\epsilon}^{\prime} = \Theta\left( \frac{\epsilon^{2}}{\vc^{2} \Log(1/\epsilon)} \right)$.
This is slightly worse (by a factor of $\vc \Log(1/\epsilon)$) than the drift rate tolerable by the 
(typically inefficient) algorithm mentioned above.
However, it does sometimes yield computationally-efficient methods.
For instance, there are known polynomial-time algorithms for the consistency problem for the classes of 
linear separators, conjunctions, and axis-aligned rectangles.

\subsection{Lower Bounds}
\label{sec:classic-lower-bound}

\citet{helmbold:94} additionally prove \emph{lower bounds} for specific concept spaces:
namely, linear separators and axis-aligned rectangles.  They specifically argue that, for 
$\C$ a concept space 
\begin{equation*}
{\rm BASIC}_{n} = \{ \cup_{i=1}^{n} [i/n,(i+a_i)/n) : \mathbf{a} \in [0,1]^{n} \}
\end{equation*}
on $[0,1]$, under $\Px$ the uniform distribution on $[0,1]$, 
for any $\epsilon \in [0,1/e^{2}]$ and $\change_{\epsilon} \geq e^{4} \epsilon^{2} / n$,
for any algorithm $\alg$, and any $T \in \nats$, there exists a choice of $\targetseq \in S_{\change_{\epsilon}}$
such that the prediction $\hat{Y}_{T}$ produced by $\alg$ at time $T$ satisfies
$\P\left( \hat{Y}_{T} \neq Y_{T} \right) > \epsilon$.
Based on this, they conclude that no $(\epsilon,S_{\change_{\epsilon}})$-tracking algorithm exists.
Furthermore, they observe that the space ${\rm BASIC}_{n}$ is embeddable in many 
commonly-studied concept spaces, including halfspaces and axis-aligned
rectangles in $\reals^{n}$, so that for $\C$ equal to either of these spaces, 
there also is no $(\epsilon,S_{\change_{\epsilon}})$-tracking algorithm.

%% file: tex-files/erm-passive.tex
This section presents a general bound on the error rate at each time,
expressed as a function of the rates of drift, which are allowed to be \emph{arbitrary}.
Most-importantly, in contrast to the methods from the literature discussed above,
the method achieving this general result is \emph{adaptive} to the drift rates,
so that it requires no information about the drift rates in advance.  This is an
appealing property, as it essentially allows the algorithm to learn under an \emph{arbitrary}
sequence $\targetseq$ of target concepts; the difficulty of the task
is then simply reflected in the resulting bounds on the error rates:
that is, faster-changing sequences of target functions result in larger bounds on
the error rates, but do not require a change in the algorithm itself.

\subsection{Adapting to a Changing Drift Rate}
\label{sec:adaptive-varying-rate}

Recall that the method yielding \eqref{eqn:hl94} (based on the work of \citealp{helmbold:94})
required access to the sequence $\changeseq$ of changes to achieve the stated guarantee
on the expected number of mistakes.  That method is based on choosing a classifier to predict $\hat{Y}_{t}$
by minimizing the number of mistakes among the previous $m_{t}$ samples, where $m_{t}$ is a value
chosen based on the $\changeseq$ sequence.  Thus, the key to modifying this algorithm to make it 
adaptive to the $\changeseq$ sequence is to determine a suitable choice of $m_{t}$ without reference
to the $\changeseq$ sequence.  The strategy we adopt here is to use the \emph{data} to determine 
an appropriate value $\hat{m}_{t}$ to use.  Roughly (ignoring logarithmic factors for now), the insight
that enables us to achieve this feat is that, 
for the $m_{t}$ used in the above strategy, one can show that $\sum_{i=t-m_{t}}^{t-1} \ind[ \target_{t}(X_{i}) \neq Y_{i} ]$
is roughly $\tilde{O}(\vc)$, and that 
making the prediction $\hat{Y}_{t}$ with \emph{any} $h \in \C$ with roughly $\tilde{O}(\vc)$ mistakes 
on these samples will suffice to obtain the stated bound on the error rate (up to logarithmic factors).
Thus, if we replace $m_{t}$ with the largest value $m$ for which $\min_{h \in \C} \sum_{i=t-m}^{t-1} \ind[ h(X_{i}) \neq Y_{i}]$
is roughly $\tilde{O}(\vc)$, then the above observation implies $m \geq m_{t}$.  This then
implies that, for $\hat{h} = \argmin_{h \in \C} \sum_{i=t-m}^{t-1} \ind[ h(X_{i}) \neq Y_{i} ]$, 
we have that $\sum_{i=t-m_{t}}^{t-1} \ind[ \hat{h}(X_{i}) \neq Y_{i} ]$ is also roughly $\tilde{O}(\vc)$,
so that the stated bound on the error rate will be achieved (aside from logarithmic factors) 
by choosing $\hat{h}_{t}$ as this classifier $\hat{h}$.
There are a few technical modifications to this argument needed to get the logarithmic factors to work out properly,
and for this reason the actual algorithm and proof below are somewhat more involved.
Specifically, consider the following algorithm (the value of the universal constant $K \geq 1$ will be specified below).

\begin{bigboxit}
0. For $T = 1,2,\ldots$\\
1. \quad Let $\hat{m}_{T} \!=\! \max\!\left\{ m \!\in\! \{1,\ldots,T\!-\!1\} : \min\limits_{h \in \C} \max\limits_{m^{\prime} \leq m} \frac{\sum_{t=T-m^{\prime}}^{T-1} \ind[h(X_{t}) \neq Y_{t}]}{\vc \Log(m^{\prime}/\vc) + \Log(1/\delta)} < K \right\}$\\
2. \quad Let $\hat{h}_{T} = \argmin\limits_{h \in \C} \max\limits_{m^{\prime} \leq \hat{m}_{T}} \frac{\sum_{t=T-m^{\prime}}^{T-1} \ind[h(X_{t}) \neq Y_{t}]}{\vc \Log(m^{\prime}/\vc) + \Log(1/\delta)}$
\end{bigboxit}

Note that the classifiers $\hat{h}_{t}$ chosen by this algorithm have no dependence on $\changeseq$, 
or indeed anything other than the data $\{(X_{i},Y_{i}) : i < t\}$, and the concept space $\C$.

\begin{theorem}
\label{thm:epst-adaptive}
Fix any $\delta \in (0,1)$, and let $\alg$ be the above algorithm.
For any sequence $\changeseq$ in $[0,1]$, for any $\Px$ and any choice of $\targetseq \in S_{\changeseq}$, 
for every $T \in \nats \setminus \{1\}$, with probability at least $1-\delta$, 
\begin{equation*}
\er_{T}\left( \hat{h}_{T} \right)
\leq O\left( \min_{1 \leq m \leq T-1} \frac{1}{m} \sum_{i=T-m}^{T-1} \sum_{j=i+1}^{T} \change_{j} + \frac{\vc \Log(m/\vc) + \Log(1/\delta)}{m} \right).
\end{equation*}
\end{theorem}

Before presenting the proof of this result, we first state a crucial lemma, which follows immediately 
from a classic result of \citet{vapnik:82,vapnik:98}, combined with the fact (from \citealp{vidyasagar:03}, Theorem 4.5)
that the VC dimension of the collection of sets $\{ \{x : h(x) \neq g(x)\} : h,g \in \C \}$ is at most $10 \vc$.

\begin{lemma}
\label{lem:vc-ratio}
There exists a universal constant $c \in [1,\infty)$ such that, 
for any class $\C$ of VC dimension $\vc$, $\forall m \in \nats$, $\forall \delta \in (0,1)$,
with probability at least $1-\delta$,
every $h,g \in \C$ have
\begin{multline*}
\left| \Px(x : h(x) \neq g(x)) - \frac{1}{m}\sum_{t=1}^{m} \ind[h(X_{t}) \neq g(X_{t})] \right|
\\ \leq c \sqrt{ \left(\frac{1}{m}\sum_{t=1}^{m} \ind[h(X_{t}) \neq g(X_{t})] \right) \frac{\vc \Log(m/\vc)+\Log(1/\delta)}{m}} 
\\ + c \frac{\vc \Log(m/\vc) + \Log(1/\delta)}{m}.
\end{multline*}
\end{lemma}

We are now ready for the proof of Theorem~\ref{thm:epst-adaptive}.
For the constant $K$ in the algorithm, we will choose $K = 145 c^{2}$,
for $c$ as in Lemma~\ref{lem:vc-ratio}.

\begin{proof}[Proof of Theorem~\ref{thm:epst-adaptive}]
Fix any $T \in \nats$ with $T \geq 2$, and define
\begin{multline*}
m_{T}^{*} = \max\left\{ m \in \{1,\ldots,T-1\} : \forall m^{\prime} \leq m, \phantom{\sum_{t=T-m^{\prime}}^{T-1}} \right.
\\ \left. \sum_{t=T-m^{\prime}}^{T-1} \ind[\target_{T}(X_{t}) \neq Y_{t}] < K ( \vc \Log(m^{\prime}/\vc) + \Log(1/\delta) )\right\}.
\end{multline*}
Note that
\begin{equation}
\label{eqn:adaptive-target-mistakes}
\sum_{t=T-m_{T}^{*}}^{T-1} \ind[\target_{T}(X_{t}) \neq Y_{t}] \leq K (\vc \Log(m_{T}^{*}/\vc) + \Log(1/\delta)),
\end{equation}
and also note that (since $\target_{T} \in \C$) $\hat{m}_{T} \geq m_{T}^{*}$, so that (by definition of $\hat{m}_{T}$ and $\hat{h}_{T}$)
\begin{equation*}
\sum_{t=T-m_{T}^{*}}^{T-1} \ind[\hat{h}_{T}(X_{t}) \neq Y_{t}] \leq K ( \vc \Log(m_{T}^{*}/\vc) + \Log(1/\delta) )
\end{equation*}
as well.
Therefore, 
\begin{align*}
\sum_{t=T-m_{T}^{*}}^{T-1} \!\!\ind[\target_{T}(X_{t}) \neq \hat{h}_{T}(X_{t})]
& \leq 
\sum_{t=T-m_{T}^{*}}^{T-1} \!\!\ind[\target_{T}(X_{t}) \neq Y_{t}]
+
\sum_{t=T-m_{T}^{*}}^{T-1} \!\!\ind[Y_{t} \neq \hat{h}_{T}(X_{t})]
\\ & \leq 
2 K ( \vc \Log(m_{T}^{*}/\vc) + \Log(1/\delta) ).
\end{align*}
Thus, by Lemma~\ref{lem:vc-ratio}, for each $m \in \nats$, 
with probability at least $1-\delta / (6 m^{2})$, if $m_{T}^{*} = m$, then
\begin{equation*}
\Px(x : \hat{h}_{T}(x) \neq \target_{T}(x))
\leq 
(2K+c \sqrt{2K} + c)  \frac{\vc \Log(m_{T}^{*}/\vc) + \Log(6(m_{T}^{*})^{2}/\delta)}{m_{T}^{*}}.
\end{equation*}
Furthermore, since 
$\Log(6(m_{T}^{*})^{2}) \leq \sqrt{2K} \vc \Log(m_{T}^{*} / \vc)$,
this is at most
\begin{equation*}
2(K+c \sqrt{2K})  \frac{\vc \Log(m_{T}^{*}/\vc) + \Log(1/\delta)}{m_{T}^{*}}.
\end{equation*}
By a union bound (over values $m \in \nats$), we have that with probability at least $1-\sum_{m=1}^{\infty} \delta/(6 m^{2}) \geq 1 - \delta/3$,
\begin{equation*}
\Px(x : \hat{h}_{T}(x) \neq \target_{T}(x))
\leq 2(K+c \sqrt{2K})  \frac{\vc \Log(m_{T}^{*}/\vc) + \Log(1/\delta)}{m_{T}^{*}}.
\end{equation*}

Let us denote
\begin{equation*}
\tilde{m}_{T} = \argmin_{m \in \{1,\ldots,T-1\}} \frac{1}{m} \sum_{i=T-m}^{T-1} \sum_{j=i+1}^{T} \change_{j} + \frac{\vc \Log(m/\vc) + \Log(1/\delta)}{m}.
\end{equation*}
Note that, for any $m^{\prime} \in \{1,\ldots,T-1\}$ and $\delta \in (0,1)$, 
if $\tilde{m}_{T} \geq m^{\prime}$, then
\begin{align*}
& \min_{m \in \{1,\ldots,T-1\}} \frac{1}{m} \sum_{i=T-m}^{T-1} \sum_{j=i+1}^{T} \change_{j} + \frac{\vc \Log(m/\vc) + \Log(1/\delta)}{m}
\\ & \geq \min_{m \in \{m^{\prime},\ldots,T-1\}} \frac{1}{m} \sum_{i=T-m}^{T-1} \sum_{j=i+1}^{T} \change_{j}
= \frac{1}{m^{\prime}} \sum_{i=T-m^{\prime}}^{T-1} \sum_{j=i+1}^{T} \change_{j},
\end{align*}
while if $\tilde{m}_{T} \leq m^{\prime}$, then
\begin{align*}
& \min_{m \in \{1,\ldots,T-1\}} \frac{1}{m} \sum_{i=T-m}^{T-1} \sum_{j=i+1}^{T} \change_{j} + \frac{\vc \Log(m/\vc) + \Log(1/\delta)}{m}
\\ & \geq \min_{m \in \{1,\ldots,m^{\prime}\}} \frac{\vc \Log(m/\vc)+\Log(1/\delta)}{m}
= \frac{\vc \Log(m^{\prime}/\vc) + \Log(1/\delta)}{m^{\prime}}.
\end{align*}
Either way, we have that
\begin{align}
& \min_{m \in \{1,\ldots,T-1\}} \frac{1}{m} \sum_{i=T-m}^{T-1} \sum_{j=i+1}^{T} \change_{j} + \frac{\vc \Log(m/\vc) + \Log(1/\delta)}{m} \notag
\\ & \geq \min\left\{ \frac{\vc \Log(m^{\prime}/\vc) + \Log(1/\delta)}{m^{\prime}}, \frac{1}{m^{\prime}} \sum_{i=T-m^{\prime}}^{T-1} \sum_{j=i+1}^{T} \change_{j} \right\}. \label{eqn:adaptive-min-lb}
\end{align}

For any $m \in \{1,\ldots,T-1\}$, 
applying Bernstein's inequality (see \citealp{boucheron:13}, equation 2.10) to the random variables $\ind[ \target_{T}(X_{i}) \neq Y_{i} ]/\vc$, $i \in \{T-m,\ldots,T-1\}$,
and again to the random variables $-\ind[\target_{T}(X_{i}) \neq Y_{i}]/\vc$, $i \in \{T-m,\ldots,T-1\}$, together with a union bound, 
we obtain that, for any $\delta \in (0,1)$, with probability at least $1 - \delta / (3m^{2})$, 
\begin{align}
& \frac{1}{m} \sum_{i=T-m}^{T-1} \Px( x : \target_{T}(x) \neq \target_{i}(x) ) \notag
\\ & {\hskip 1cm}- \sqrt{ \left( \frac{1}{m} \sum_{i=T-m}^{T-1} \Px( x : \target_{T}(x) \neq \target_{i}(x) ) \right) \frac{2\ln(3m^{2}/\delta)}{m} } \notag
\\ & < \frac{1}{m} \sum_{i=T-m}^{T-1} \ind[ \target_{T}(X_{i}) \neq Y_{i} ] \notag
\\ & < \frac{1}{m} \sum_{i=T-m}^{T-1} \Px( x : \target_{T}(x) \neq \target_{i}(x) ) \notag
\\ & {\hskip 1cm}+ \max\begin{cases}
\sqrt{ \left( \frac{1}{m} \sum_{i=T-m}^{T-1} \Px( x : \target_{T}(x) \neq \target_{i}(x) ) \right) \frac{4\ln(3m^{2}/\delta)}{m} } 
\\\frac{(4/3)\ln(3m^{2}/\delta)}{m} \end{cases}.\label{eqn:adaptive-empirical-ub}
\end{align}
The left inequality implies that
\begin{equation*}
\frac{1}{m} \!\sum_{i=T-m}^{T-1}\!\!\! \Px( x \!:\! \target_{T}(x) \neq \target_{i}(x) )
\leq \max\left\{ \frac{2}{m} \!\sum_{i=T-m}^{T-1} \!\!\!\ind[ \target_{T}(X_{i}) \neq Y_{i} ], \frac{8\ln(3m^{2}/\delta)}{m} \right\}\!.
\end{equation*}
Plugging this into the right inequality in \eqref{eqn:adaptive-empirical-ub}, we obtain that
\begin{multline*}
\frac{1}{m} \sum_{i=T-m}^{T-1} \ind[ \target_{T}(X_{i}) \neq Y_{i} ]
< \frac{1}{m} \sum_{i=T-m}^{T-1} \Px( x : \target_{T}(x) \neq \target_{i}(x) ) 
\\ + \max\left\{ \sqrt{ \left(\frac{1}{m} \sum_{i=T-m}^{T-1} \ind[ \target_{T}(X_{i}) \neq Y_{i} ] \right) \frac{8\ln(3m^{2}/\delta)}{m} }, \frac{\sqrt{32}\ln(3m^{2}/\delta)}{m} \right\}.
\end{multline*}
By a union bound, this holds simultaneously for all $m \in \{1,\ldots,T-1\}$ with probability at least $1-\sum_{m = 1}^{T-1} \delta / (3m^{2}) > 1 - (2/3)\delta$.
Note that, on this event,
we obtain
\begin{multline*}
\frac{1}{m} \sum_{i=T-m}^{T-1} \Px( x : \target_{T}(x) \neq \target_{i}(x) ) 
> 
\frac{1}{m} \sum_{i=T-m}^{T-1} \ind[ \target_{T}(X_{i}) \neq Y_{i} ]
\\ - \max\left\{ \sqrt{ \left(\frac{1}{m} \sum_{i=T-m}^{T-1} \ind[ \target_{T}(X_{i}) \neq Y_{i} ] \right) \frac{8\ln(3m^{2}/\delta)}{m} }, \frac{\sqrt{32}\ln(3m^{2}/\delta)}{m} \right\}.
\end{multline*}
In particular, taking $m = m_{T}^{*}$, and invoking maximality of $m_{T}^{*}$, if $m_{T}^{*} < T-1$, the right hand side is at least
\begin{equation*}
(K - 6c\sqrt{K}) \frac{\vc \Log(m_{T}^{*}/\vc) + \Log(1/\delta)}{m_{T}^{*}}.
\end{equation*}

Since $\frac{1}{m} \sum_{i=T-m}^{T-1} \sum_{j=i+1}^{T} \change_{j} \geq \frac{1}{m} \sum_{i=T-m}^{T-1} \Px( x : \target_{T}(x) \neq \target_{i}(x) )$,
taking $K = 145 c^{2}$,
we have that with probability at least $1-\delta$, if $m_{T}^{*} < T-1$, then 
\begin{align*}
& 10(K+c \sqrt{2K})\min_{m \in \{1,\ldots,T-1\}} \frac{1}{m} \sum_{i=T-m}^{T-1} \sum_{j=i+1}^{T} \change_{j} + \frac{\vc \Log(m/\vc)+\Log(1/\delta)}{m}
\\ & \geq 
10(K+c \sqrt{2K})\min\left\{ \frac{\vc \Log(m_{T}^{*}/\vc)+\Log(1/\delta)}{m_{T}^{*}}, \frac{1}{m_{T}^{*}} \sum_{i=T-m_{T}^{*}}^{T-1} \sum_{j=i+1}^{T} \change_{j} \right\}
\\ & \geq 
10(K+c \sqrt{2K})\frac{\vc \Log(m_{T}^{*}/\vc) + \Log(1/\delta)}{m_{T}^{*}}
\\ & \geq \Px(x : \hat{h}_{T}(x) \neq \target_{T}(x)).
\end{align*}
Furthermore, if $m_{T}^{*} = T-1$, then we trivially have (on the same $1-\delta$ probability event as above)
\begin{align*}
& 10(K+c \sqrt{2K})\min_{m \in \{1,\ldots,T-1\}} \frac{1}{m} \sum_{i=T-m}^{T-1} \sum_{j=i+1}^{T} \change_{j} + \frac{\vc \Log(m/\vc)+\Log(1/\delta)}{m}
\\ & \geq 10(K+c \sqrt{2K}) \min_{m \in \{1,\ldots,T-1\}} \frac{\vc \Log(m/\vc)+\Log(1/\delta)}{m}
\\ & = 10(K+c \sqrt{2K}) \frac{\vc \Log((T-1)/\vc)+\Log(1/\delta)}{T-1}
\\ & = 10(K+c \sqrt{2K})\frac{\vc \Log(m_{T}^{*}/\vc) + \Log(1/\delta)}{m_{T}^{*}}
\geq \Px(x : \hat{h}_{T}(x) \neq \target_{T}(x)).
\end{align*}
\qed
\end{proof}

\subsection{Conditions Guaranteeing a Sublinear Number of Mistakes}
\label{sec:sublinear}
\input{tex-files/sublinear.tex}

%% file: tex-files/sublinear.tex
One immediate implication of Theorem~\ref{thm:epst-adaptive} is 
that, if the sum of $\change_{t}$ values grows sublinearly, 
then there exists an algorithm achieving an expected number of mistakes
growing sublinearly in the number of predictions.
Formally, we have the following corollary.  

\begin{corollary}
\label{cor:sublinear-implies-sublinear}
If $\sum_{t=1}^{T} \change_{t} = o(T)$, then there exists an algorithm $\alg$ 
such that, for every $\Px$ and every choice of $\targetseq \in S_{\changeseq}$, 
\begin{equation*}
\E\left[ \sum_{t=1}^{T} \ind\left[ \hat{Y}_{t} \neq Y_{t} \right] \right] = o(T).
\end{equation*}
\end{corollary}
\begin{proof}
For every $T \in \nats$ with $T \geq 2$, 
let 
\begin{equation*}
\tilde{m}_{T} = \argmin_{1 \leq m \leq T-1} \frac{1}{m} \sum_{i=T-m}^{T-1} \sum_{j=i+1}^{T} \change_{j} + \frac{\vc \Log(m/\vc) + \Log(1/\delta_{T})}{m},
\end{equation*}
and define $\delta_{T} = \frac{1}{\tilde{m}_{T}}$.
Then consider running the algorithm $\alg$ from Theorem~\ref{thm:epst-adaptive},
except that in choosing $\hat{m}_{T}$ and $\hat{h}_{T}$ for each $T$, 
we use the above value $\delta_{T}$ in place of $\delta$.
Then Theorem~\ref{thm:epst-adaptive} implies that, for each $T$,
with probability at least $1-\delta_{T}$, 
\begin{equation*}
\er_{T}(\hat{h}_{T}) \leq O\left( \min_{1 \leq m \leq T-1} \frac{1}{m} \sum_{i=T-m}^{T-1} \sum_{j=i+1}^{T} \change_{j} + \frac{\vc \Log(m/\vc) + \Log(1/\delta_{T})}{m} \right).
\end{equation*}
Since $\er_{T}(\hat{h}_{T}) \leq 1$, this implies that
\begin{align*}
& \P\left( \hat{Y}_{T} \neq Y_{T} \right)
= \E\left[ \er_{T}(\hat{h}_{T}) \right]
\\ & \leq O\left( \min_{1 \leq m \leq T-1} \frac{1}{m} \sum_{i=T-m}^{T-1} \sum_{j=i+1}^{T} \change_{j} + \frac{\vc \Log(m/\vc) + \Log(1/\delta_{T})}{m} \right) + \delta_{T}
\\ & = O\left( \min_{1 \leq m \leq T-1} \frac{1}{m} \sum_{i=T-m}^{T-1} \sum_{j=i+1}^{T} \change_{j} + \frac{\vc \Log(m/\vc) + \Log(m)}{m} \right),
\end{align*}
and since $x \mapsto x \Log(m/x)$ is nondecreasing for $x \geq 1$, $\Log(m) \leq \vc \Log(m/\vc)$, so that this last expression is
\begin{equation*}
O\left( \min_{1 \leq m \leq T-1} \frac{1}{m} \sum_{i=T-m}^{T-1} \sum_{j=i+1}^{T} \change_{j} + \frac{\vc \Log(m/\vc)}{m} \right).
\end{equation*}

Now note that, for any $t \in \nats$ and $m \in \{1,\ldots,t-1\}$, 
\begin{equation}
\label{eqn:sublinear-eqn-1}
\frac{1}{m}\sum_{s=t-m}^{t-1} \sum_{r=s+1}^{t} \change_{r} 
\leq \frac{1}{m}\sum_{s=t-m}^{t-1} \sum_{r=t-m+1}^{t} \change_{r} 
= \sum_{r=t-m+1}^{t} \change_{s}.
\end{equation}
Let $\beta_{t}(m) = \max\left\{\sum_{r=t-m+1}^{t} \change_{r} , \frac{\vc\Log(m/\vc)}{m}\right\}$,
and note that 
$\sum_{r=t-m+1}^{t} \change_{r} + \frac{\vc\Log(m/\vc)}{m} \leq 2 \beta_{t}(m)$.
Thus, combining the above with \eqref{eqn:sublinear-eqn-1}, linearity of expectations, and 
the fact that the probability of a mistake on a given round is at most $1$, we obtain
\begin{equation*}
\E\left[ \sum_{t=1}^{T} \ind\left[ \hat{Y}_{t} \neq Y_{t} \right] \right]
= O\left(\sum_{t=1}^{T} \min_{m \in \{1,\ldots,t-1\}} \beta_{t}(m) \land 1 \right).
\end{equation*}
Fixing any $M \in \nats$, we have that for any $T > M$, 
\begin{align*}
& \sum_{t=1}^{T} \min_{m \in \{1,\ldots,t-1\}} \beta_{t}(m) \land 1
\leq M + \sum_{t=M+1}^{T} \beta_{t}(M) \land 1
\\ & \leq M + \sum_{t=M+1}^{T} \ind\left[ \frac{\vc \Log(M/\vc)}{M} \geq \sum_{r=t-M+1}^{t} \change_{r} \right] \frac{\vc \Log(M/\vc)}{M}
\\ & {\hskip 1cm}+ \sum_{t=M+1}^{T} \ind\left[ \sum_{r=t-M+1}^{t} \change_{r} > \frac{\vc \Log(M/\vc)}{M} \right]
\\ & \leq M + \frac{\vc \Log(M/\vc)}{M} T + \sum_{t=M+1}^{T} \frac{M}{\vc \Log(M/\vc)} \sum_{r=t-M+1}^{t} \change_{r}
\\ & = \frac{\vc \Log(M/\vc)}{M} T + g_{M}(T),
\end{align*}
where $g_{M}$ is a function satisfying $g_{M}(T) = o(T)$ (holding $M$ fixed).
Since this is true of \emph{any} $M \in \nats$, we have that 
\begin{align*}
\lim_{T \to \infty} \frac{1}{T} \sum_{t=1}^{T} \min_{m \in \{1,\ldots,t-1\}} \beta_{t}(m) \land 1
& \leq \lim_{M \to \infty} \lim_{T \to \infty} \frac{\vc \Log(M/\vc)}{M} + \frac{g_{M}(T)}{T}
\\ & = \lim_{M \to \infty} \frac{\vc \Log(M/\vc)}{M}
= 0, 
\end{align*}
so that $\E\left[ \sum_{t=1}^{T} \ind\left[ \hat{Y}_{t} \neq Y_{t} \right] \right] = o(T)$, as claimed.
\qed
\end{proof}

For many concept spaces of interest, the condition $\sum_{t=1}^{T} \change_{t} = o(T)$ in Corollary~\ref{cor:sublinear-implies-sublinear}
is also a \emph{necessary} condition for \emph{any} algorithm to guarantee a sublinear number of mistakes.  
For simplicity, we will establish this for the class of \emph{homogeneous linear separators} on $\reals^{2}$,
with $\Px$ the uniform distribution on the unit circle, in the following theorem.  This can easily be extended to many other spaces, 
including higher-dimensional linear separators or axis-aligned rectangles in $\reals^{k}$, by embedding an analogous setup into those spaces.

\begin{theorem}
\label{thm:sublinear-necessary}
If $\X = \{ x \in \reals^{2} : \|x\|=1 \}$, $\Px$ is ${\rm Uniform}(\X)$, and $\C = \{ x \mapsto 2 \ind[ w \cdot x \geq 0 ] - 1 : w \in \reals^{2}, \|w\|=1 \}$ is the class of homogeneous linear separators,
then for any sequence $\changeseq$ in $[0,1]$, there exists an algorithm $\alg$ such that 
$\E\left[ \sum_{t=1}^{T} \ind\left[ \hat{Y}_{t} \neq Y_{t} \right] \right] = o(T)$
for \emph{every} choice of $\targetseq \in S_{\changeseq}$
\emph{if and only if} $\sum_{t=1}^{T} \change_{t} = o(T)$.
\end{theorem}
\begin{proof}
The ``if'' part follows immediately from Corollary~\ref{cor:sublinear-implies-sublinear}.
For the ``only if'' part, suppose $\changeseq$ is such that $\sum_{t=1}^{T} \change_{t} \neq o(T)$.
It suffices to argue that for any algorithm $\alg$, there exists a choice of $\targetseq \in S_{\changeseq}$ 
for which $\E\left[ \sum_{t=1}^{T} \ind\left[ \hat{Y}_{t} \neq Y_{t} \right] \right] \neq o(T)$.  Toward this end, fix any algorithm $\alg$.
We proceed by the probabilistic method, constructing a \emph{random} sequence $\targetseq \in S_{\changeseq}$.
Let $B_{1},B_{2},\ldots$ be independent Bernoulli($1/2$) random variables (also independent from the unlabeled data $X_{1},X_{2},\ldots$).
We define the sequence $\targetseq$ inductively.  For simplicity, we will represent each classifier in \emph{polar} coordinates, 
writing $h_{\phi}$ (for $\phi \in \reals$) to denote the classifier that, for $x = (x_{1},x_{2})$, 
classifies $x$ as $h_{\phi}(x) = 2 \ind[ x_{1} \cos(\phi) + x_{2} \sin(\phi) \geq 0 ] - 1$; note that 
$h_{\phi} = h_{\phi + 2\pi}$ for every $\phi \in \reals$.
As a base case, start by defining a function $\target_{0} = h_{0}$, and letting $\phi_{0} = 0$.
Now for any $t \in \nats$, supposing $\target_{t-1}$ is already defined to be $h_{\phi_{t-1}}$,
we define $\phi_{t} = \phi_{t-1} + \min\{\change_{t},1/2\} \pi B_{t}$, and $\target_{t} = h_{\phi_{t}}$.
Note that $\Px( x : \target_{t}(x) \neq \target_{t-1}(x) ) = \min\{\change_{t},1/2\}$ for every $t \in \nats$,
so that this inductively defines a (random) choice of $\targetseq \in S_{\changeseq}$.

For each $t \in \nats$, let $Y_{t} = \target_{t}(X_{t})$.
Now fix any algorithm $\alg$, and consider the sequence $\hat{Y}_{t}$ of predictions
the algorithm makes for points $X_{t}$, when the target sequence $\targetseq$ is chosen as above.
Then note that, for any $t \in \nats$, since $\hat{Y}_{t}$ and $B_{t}$ are independent, 
\begin{align*}
\P\left( \hat{Y}_{t} \neq Y_{t} \right)
& \geq \E\left[ \P\left( \hat{Y}_{t} \neq Y_{t} \middle| \hat{Y}_{t},\phi_{t-1} \right) \right]
\\ & \geq \E\left[ \frac{1}{2} \P\left( h_{\phi_{t-1}+\min\{\change_{t},1/2\}\pi}(X_{t}) \neq h_{\phi_{t-1}-\min\{\change_{t},1/2\}\pi}(X_{t}) \middle| \phi_{t-1} \right) \right].
\end{align*}
Furthermore, since $\min\{\change_{t},1/2\} \pi \leq \pi/2$, the regions $\{x : h_{\phi_{t-1}+\min\{\change_{t},1/2\}\pi}(x) \neq h_{\phi_{t-1}}(x) \}$ 
and $\{x : h_{\phi_{t-1}-\min\{\change_{t},1/2\}\pi}(x) \neq h_{\phi_{t-1}}(x) \}$ have zero-probability overlap (indeed, are disjoint if $\change_{t} < 1/2$),
the above equals $\min\{\change_{t},1/2\}$.

By Fatou's lemma, linearity of expectations, and the law of total expectation, we have that
\begin{align*}
\E\left[ \limsup_{T \to \infty} \frac{1}{T} \E\left[ \sum_{t=1}^{T} \ind[ \hat{Y}_{t} \neq Y_{t} ] \middle| \targetseq \right] \right]
& \geq \limsup_{T \to \infty} \frac{1}{T} \sum_{t=1}^{T} \P\left( \hat{Y}_{t} \neq Y_{t} \right)
\\ & \geq \limsup_{T \to \infty} \frac{1}{T} \sum_{t=1}^{T} \min\{\change_{t},1/2\}.
\end{align*}
Since $\sum_{t=1}^{T} \change_{t} \neq o(T)$, the rightmost expression is strictly greater than zero.
Thus, it must be that, with probility strictly greater than $0$, 
\begin{equation*}
\limsup_{T \to \infty} \frac{1}{T} \E\left[ \sum_{t=1}^{T} \ind[ \hat{Y}_{t} \neq Y_{t} ] \middle| \targetseq \right] > 0.
\end{equation*}
In particular, this implies that there exists a (nonrandom) choice of the sequence $\targetseq \in S_{\changeseq}$
for which $\E\left[ \sum_{t=1}^{T} \ind\left[ \hat{Y}_{t} \neq Y_{t} \right] \right] \neq o(T)$.
Since this holds for \emph{any} choice of the algorithm $\alg$, this completes the proof.
\qed
\end{proof}

%% file: tex-files/halfspaces.tex
In this section, we suppose $\change_{t} = \change$ for every $t \in \nats$, for a fixed constant $\change > 0$,
and we consider the special case of learning homogeneous linear separators in $\reals^{k}$ under a uniform distribution 
on the origin-centered unit sphere.
In this case, the analysis of \citet{helmbold:94} mentioned in Section~\ref{sec:classic-consistency} implies that it is possible to achieve a 
bound on the error rate that is $\tilde{O}(d \sqrt{\change})$, 
using an algorithm that runs in time $\poly(d,1/\change,\log(1/\delta))$ (and independent of $t$) for each prediction.
This also implies that it is possible to achieve expected number of mistakes among $T$ predictions that is $\tilde{O}(d \sqrt{\change}) \times T$.
\citet{min_concept}\footnote{This 
work in fact studies a much broader model of drift, which in fact allows the distribution $\Px$ to vary with time as well.  However, this $\tilde{O}((d \change)^{1/4}) \times T$ result can be obtained
from their more-general theorem by calculating the various parameters for this particular setting.}
have since proven that a variant of the Perceptron algorithm is capable of achieving an expected number of mistakes $\tilde{O}( (d \change)^{1/4} ) \times T$.

Below, we improve on this result by showing that there exists an efficient algorithm that achieves a 
bound on the error rate that is $\tilde{O}(\sqrt{d \change})$, 
as was possible with the inefficient algorithm of \citet{helmbold:94,long:99} mentioned in Section~\ref{sec:classic-constant-drift}.
This leads to a bound on the expected number of mistakes that is $\tilde{O}(\sqrt{d \change}) \times T$.
Furthermore, our approach also allows us to present the method as an \emph{active learning}
algorithm, and to bound the expected number of queries, as a function of the 
number of samples $T$, by $\tilde{O}(\sqrt{d \change}) \times T$.
The technique is based on a modification of the algorithm of \citet{helmbold:94},
replacing an empirical risk minimization step with (a modification of) the computationally-efficient algorithm of \citet{awasthi:13}.

Formally, define the class of homogeneous linear separators as the set of classifiers 
$h_{w} : \reals^{d} \to \{-1,+1\}$, for $w \in \reals^{d}$ with $\|w\|=1$,
such that $h_{w}(x) = \sign( w \cdot x )$ for every $x \in \reals^{d}$.

\subsection{An Improved Guarantee for a Polynomial-Time Algorithm}
\label{sec:efficient-linsep}

We have the following result.

\begin{theorem}
\label{thm:linsep-uniform}
When $\C$ is the space of homogeneous linear separators (with $d \geq 4$)
and $\Px$ is the uniform distribution on the surface of 
the origin-centered unit sphere in $\reals^{d}$, 
for any fixed $\change > 0$, 
for any $\delta \in (0,1/e)$, 
there is an algorithm that runs in time $\poly(d,1/\change,\log(1/\delta))$ for each time $t$,
such that for any $\targetseq \in S_{\change}$, 
for every sufficiently large $t \in \nats$, with probability at least $1-\delta$, 
\begin{equation*}
\er_{t}(\hat{h}_{t}) = O\left( \sqrt{\change d \log\left(\frac{1}{\delta}\right) } \right).
\end{equation*}
Also, running this algorithm with $\delta = \sqrt{\change d} \land 1/e$, 
the expected number of mistakes among the first $T$ instances is
$O\left( \sqrt{ \change d \log\left(\frac{1}{\change d}\right) } T \right)$.
Furthermore, the algorithm can be run as an \emph{active learning} algorithm,
in which case, for this choice of $\delta$, the expected number of labels 
requested by the algorithm among the first $T$ instances is
$O\left( \sqrt{\change d} \log^{3/2}\left(\frac{1}{\change d}\right) T \right)$.
\end{theorem}

We first state the algorithm used to obtain this result.  It is primarily based on a 
margin-based learning strategy of \citet{awasthi:13}, combined with an initialization 
step based on a modified Perceptron rule from \citet{stream_perceptron,min_concept}.
For $\tau > 0$ and $x \in \reals$, define $\ell_{\tau}(x) = \max\left\{0, 1 - \frac{x}{\tau}\right\}$.
Consider the following algorithm and subroutine; 
parameters $\delta_{k}$, $m_{k}$, $\tau_{k}$, $r_{k}$, $b_{k}$, $\alpha$, and $\kappa$
will all be specified in the context of the proof; we suppose $M = \sum_{k=0}^{\lceil \log_{2}(1/\alpha) \rceil} m_{k}$.

\begin{bigboxit}
Algorithm: DriftingHalfspaces\\
0. Let $\tilde{h}_{0}$ be an arbitrary classifier in $\C$\\
1. For $i = 1,2,\ldots$\\
2. \quad $\tilde{h}_{i} \gets {\rm ABL}(M (i-1), \tilde{h}_{i-1})$\\
\end{bigboxit}

\begin{bigboxit}
Subroutine: ${\rm ModPerceptron}(t,\tilde{h})$\\
0. Let $w_{t}$ be any element of $\reals^{d}$ with $\|w_{t}\| = 1$\\
1. For $m = t+1,t+2,\ldots,t+m_{0}$\\
2. \quad Choose $\hat{h}_{m} = \tilde{h}$ (i.e., predict $\hat{Y}_{m} = \tilde{h}(X_{m})$ as the prediction for $Y_{m}$)\\
3. \quad Request the label $Y_{m}$\\
4. \quad If $h_{w_{m-1}}(X_{m}) \neq Y_{m}$\\
5. \qquad $w_{m} \gets w_{m-1} - 2(w_{m-1} \cdot X_{m}) X_{m}$\\
6. \quad Else $w_{m} \gets w_{m-1}$\\
7. Return $w_{t+m_{0}}$
\end{bigboxit}

\begin{bigboxit}
Subroutine: ${\rm ABL}(t,\tilde{h})$\\
0. Let $w_{0}$ be the return value of ${\rm ModPerceptron}(t,\tilde{h})$\\
1. For $k = 1,2,\ldots,\lceil \log_{2}(1/\alpha) \rceil$\\
2. \quad $W_{k} \gets \{\}$\\
3. \quad For $s = t + \sum_{j=0}^{k-1} m_{j} + 1, \ldots, t + \sum_{j=0}^{k} m_{j}$\\
4. \qquad Choose $\hat{h}_{s} = \tilde{h}$ (i.e., predict $\hat{Y}_{s} = \tilde{h}(X_{s})$ as the prediction for $Y_{s}$)\\
5. \qquad If $|w_{k-1} \cdot X_{s}| \leq b_{k-1}$, Request label $Y_{s}$ and let $W_{k} \gets W_{k} \cup \{(X_{s},Y_{s})\}$\\
6. \quad Find $v_{k} \in \reals^{d}$ with $\|v_{k} - w_{k-1}\| \leq r_{k}$, $0 < \|v_{k}\| \leq 1$, 
and\\ {\hskip 7mm}$\sum\limits_{(x,y) \in W_{k}} \ell_{\tau_{k}}(y (v_{k} \cdot x)) \leq \inf\limits_{v : \|v-w_{k-1}\| \leq r_{k}} \sum\limits_{(x,y) \in W_{k}} \ell_{\tau_{k}}(y (v \cdot x)) + \kappa |W_{k}|$\\
7. \quad Let $w_{k} = \frac{1}{\|v_{k}\|} v_{k}$\\
8. Return $h_{w_{\lceil \log_{2}(1/\alpha) \rceil-1}}$
\end{bigboxit}

Before stating the proof, we have a few additional lemmas that will be needed.
The following result for ${\rm ModPerceptron}$ was proven by \citet{min_concept}.

\begin{lemma}
\label{lem:perceptron}
Suppose $\change < \frac{1}{512}$.
Consider the values $w_{m}$ obtained during the execution of ${\rm ModPerceptron}(t,\tilde{h})$.
$\forall m \in \{t+1,\ldots, t+ m_{0}\}$, $\Px(x : h_{w_{m}}(x) \neq h_{m}^{*}(x)) \leq \Px(x : h_{w_{m-1}}(x) \neq h_{m}^{*}(x))$.
Furthermore, letting $c_{1} = \frac{\pi^{2}}{d \cdot 400 \cdot 2^{15}}$, if
$\Px(x : h_{w_{m-1}}(x) \neq h_{m}^{*}(x)) \geq 1/32$,
then with probability at least $1/64$,
$\Px(x : h_{w_{m}}(x) \neq h_{m}^{*}(x)) \leq (1 - c_{1}) \Px(x : h_{w_{m-1}}(x) \neq h_{m}^{*}(x))$.
\end{lemma}

This implies the following.

\begin{lemma}
\label{lem:perceptron-init}
Suppose $\change \leq \frac{\pi^{2}}{400 \cdot 2^{27} (d+\ln(4/\delta))}$.
For $m_{0} = \max\{\lceil 128 (1/c_{1}) \ln(32) \rceil,$ $\lceil 512 \ln(\frac{4}{\delta}) \rceil \}$,
with probability at least $1-\delta/4$, 
${\rm ModPerceptron}(t,\tilde{h})$ returns a vector $w$ with 
$\Px(x : h_{w}(x) \neq h_{t+m_{0}+1}^{*}(x)) \leq 1/16$.
\end{lemma}
\begin{proof}
By Lemma~\ref{lem:perceptron} and a union bound, in general we have
\begin{equation}
\label{eqn:perceptron-weak-update}
\Px(x : h_{w_{m}}(x) \neq h_{m+1}^{*}(x)) \leq \Px(x : h_{w_{m-1}}(x) \neq h_{m}^{*}(x)) + \change.
\end{equation}
Furthermore, if $\Px(x : h_{w_{m-1}}(x) \neq h_{m}^{*}(x)) \geq 1/32$, 
then wth probability at least $1/64$,
\begin{equation}
\label{eqn:perceptron-strong-update}
\Px(x : h_{w_{m}}(x) \neq h_{m+1}^{*}(x)) \leq (1-c_{1}) \Px(x : h_{w_{m-1}}(x) \neq h_{m}^{*}(x)) + \change.
\end{equation}
In particular, this implies that the number $N$ of values $m \in \{t+1,\ldots,t+m_{0}\}$ with either
$\Px(x : h_{w_{m-1}}(x) \neq h_{m}^{*}(x)) < 1/32$ or $\Px(x : h_{w_{m}}(x) \neq h_{m+1}^{*}(x)) \leq (1-c_{1}) \Px(x : h_{w_{m-1}}(x) \neq h_{m}^{*}(x)) + \change$
is lower-bounded by a ${\rm Binomial}(m,1/64)$ random variable.
Thus, a Chernoff bound implies that with probability at least $1 - \exp\{ - m_{0} / 512 \} \geq 1 - \delta/4$,
we have $N \geq m_{0} / 128$.  Suppose this happens.

Since $\change m_{0} \leq 1/32$, if any $m \in \{t+1,\ldots,t+m_{0}\}$ has $\Px(x : h_{w_{m-1}}(x) \neq h_{m}^{*}(x)) < 1/32$,
then inductively applying \eqref{eqn:perceptron-weak-update} implies that
$\Px(x : h_{w_{t+m_{0}}}(x) \neq h_{t+m_{0}+1}^{*}(x)) \leq 1/32 + \change m_{0} \leq 1/16$.
On the other hand, if all $m \in \{t+1,\ldots,t+m_{0}\}$ have $\Px(x : h_{w_{m-1}}(x) \neq h_{m}^{*}(x)) \geq 1/32$,
then in particular we have $N$ values of $m \in \{t+1,\ldots,t+m_{0}\}$ satisfying \eqref{eqn:perceptron-strong-update}.
Combining this fact with \eqref{eqn:perceptron-weak-update} inductively, we have that
\begin{multline*}
\Px(x : h_{w_{t+m_{0}}}(x) \neq h_{t+m_{0}+1}^{*}(x)) 
\leq (1-c_{1})^{N} \Px(x : h_{w_{t}}(x) \neq h_{t+1}^{*}(x)) + \change m_{0} 
\\ \leq (1-c_{1})^{(1/c_{1}) \ln(32) } \Px(x : h_{w_{t}}(x) \neq h_{t+1}^{*}(x)) + \change m_{0} 
\leq \frac{1}{32} + \change m_{0} 
\leq \frac{1}{16}.
\end{multline*}
\qed
\end{proof}

Next, we consider the execution of ${\rm ABL}(t,\tilde{h})$, and let the sets $W_{k}$ be as in that execution.
We will denote by $w^{*}$ the weight vector with $\|w^{*}\|=1$ such that $h_{t+m_{0}+1}^{*} = h_{w^{*}}$.
Also denote by $M_{1} = M-m_{0}$.

The proof relies on a few results proven in the work of \citet{awasthi:13}, which we summarize in the following lemmas.
Although the results were proven in a slightly different setting in that work (namely, agnostic learning under a fixed joint distribution),
one can easily verify that their proofs remain valid in our present context as well.

\begin{lemma}
\label{lem:denoised-risk}
\citep{awasthi:13}
Fix any $k \in \{1,\ldots,\lceil \log_{2}(1/\alpha) \rceil\}$.
For a universal constant $c_{7} > 0$, suppose $b_{k-1} = c_{7} 2^{1-k} / \sqrt{d}$, 
and let $z_{k} = \sqrt{r_{k}^{2}/(d-1) + b_{k-1}^{2}}$.
For a universal constant $c_{1} > 0$, if $\|w^{*} - w_{k-1}\| \leq r_{k}$,
\begin{multline*}
{\hskip -3mm}\left| \E\!\left[ \sum_{(x,y) \in W_{k}} \ell_{\tau_{k}}(|w^{*} \cdot x|) \Big| w_{k-1}, |W_{k}| \right] 
- \E\!\left[ \sum_{(x,y) \in W_{k}} \ell_{\tau_{k}}(y (w^{*} \cdot x)) \Big| w_{k-1}, |W_{k}| \right] \right|
\\ \leq c_{1} |W_{k}| \sqrt{2^{k} \change M_{1}} \frac{z_{k}}{\tau_{k}}.
\end{multline*}
\end{lemma}

\begin{lemma}
\label{lem:margin-error-concentration}
\citep{balcan:13}
For any $c > 0$, there is a constant $c^{\prime} > 0$ depending only on $c$ (i.e., not depending on $d$)
such that, for any $u,v \in \reals^{d}$ with $\|u\|=\|v\|=1$, letting $\sigma = \Px(x : h_{u}(x) \neq h_{v}(x))$,
if $\sigma < 1/2$, then
\begin{equation*}
\Px\left( x : h_{u}(x) \neq h_{v}(x) \text{ and } |v \cdot x| \geq c^{\prime} \frac{\sigma}{\sqrt{d}} \right) \leq c \sigma.
\end{equation*}
\end{lemma}

The following is a well-known lemma concerning concentration around the equator for the uniform distribution (see e.g., \citealp{stream_perceptron,balcan:07,awasthi:13});
for instance, it easily follows from the formulas for the area in a spherical cap derived by \citet{li:11}.

\begin{lemma}
\label{lem:uniform-P-concentration}
For any constant $C > 0$, there are constants $c_{2},c_{3} > 0$ depending only on $C$ (i.e., independent of $d$) such that,
for any $w \in \reals^{d}$ with $\|w\|=1$, $\forall \gamma \in [0, C/\sqrt{d}]$,
\begin{equation*}
c_{2} \gamma \sqrt{d} \leq \Px\left( x : |w \cdot x| \leq \gamma \right) \leq c_{3} \gamma \sqrt{d}.
\end{equation*}
\end{lemma}

Based on this lemma, \citet{awasthi:13} prove the following.

\begin{lemma}
\label{lem:opt-margin-loss}
\citep{awasthi:13}
For $X \sim \Px$, for any $w \in \reals^{d}$ with $\|w\|=1$, for any $C > 0$ and $\tau, b \in [0,C/\sqrt{d}]$,
for $c_{2},c_{3}$ as in Lemma~\ref{lem:uniform-P-concentration},
\begin{equation*}
\E\left[ \ell_{\tau}( |w^{*} \cdot X| ) \Big| |w \cdot X| \leq b \right] \leq \frac{c_{3} \tau}{c_{2} b}.
\end{equation*}
\end{lemma}

The following is a slightly stronger version of a result of \citet{awasthi:13} (specifically, 
the size of $m_{k}$, and consequently the bound on $|W_{k}|$, are both improved by a factor of $d$ 
compared to the original result).

\begin{lemma}
\label{lem:margin-error-bound}
Fix any $\delta \in (0,1/e)$.
For universal constants $c_{4},c_{5},c_{6},c_{7},c_{8},c_{9},c_{10} \in (0,\infty)$, 
for an appropriate choice of $\kappa \in (0,1)$ (a universal constant),
if $\alpha = c_{9} \sqrt{\change d \log\left(\frac{1}{\kappa\delta}\right)}$,
for every $k \in \{1,\ldots,\lceil \log_{2}(1/\alpha) \rceil\}$,
if $b_{k-1} = c_{7} 2^{1-k} / \sqrt{d}$, $\tau_{k} = c_{8} 2^{-k} / \sqrt{d}$, $r_{k} = c_{10} 2^{-k}$, $\delta_{k} = \delta / (\lceil \log_{2}(4/\alpha) \rceil - k)^{2}$,
and $m_{k} = \left\lceil c_{5} \frac{2^{k}}{\kappa^{2}} d \log\left(\frac{1}{\kappa\delta_{k}} \right)\right\rceil$,
and if $\Px(x : h_{w_{k-1}}(x) \neq h_{w^{*}}(x)) \leq 2^{-k-3}$,
then with probability at least $1-(4/3)\delta_{k}$, 
$|W_{k}| \leq c_{6} \frac{1}{\kappa^{2}} d \log\left(\frac{1}{\kappa\delta_{k}}\right)$
and
$\Px(x : h_{w_{k}}(x) \neq h_{w^{*}}(x)) \leq 2^{-k-4}$.
\end{lemma}
\begin{proof}
By Lemma~\ref{lem:uniform-P-concentration}, and a Chernoff and union bound, 
for an appropriately large choice of $c_{5}$ and any $c_{7} > 0$,
letting $c_{2},c_{3}$ be as in Lemma~\ref{lem:uniform-P-concentration} (with $C=c_{7} \lor (c_{8}/2)$),
with probability at least $1-\delta_{k}/3$, 
\begin{equation}
\label{eqn:Wk-bounds}
c_{2} c_{7} 2^{-k} m_{k}
\leq |W_{k}| \leq
4 c_{3} c_{7} 2^{-k} m_{k}.
\end{equation}
The claimed upper bound on $|W_{k}|$ follows from this second inequality.

Next note that, if $\Px(x : h_{w_{k-1}}(x) \neq h_{w^{*}}(x)) \leq 2^{-k-3}$, 
then 
\begin{equation*}
\max\{ \ell_{\tau_{k}}(y (w^{*} \cdot x)) :  x \in \reals^{d}, |w_{k-1} \cdot x| \leq b_{k-1}, y \in \{-1,+1\} \} \leq c_{11} \sqrt{d}
\end{equation*}
for some universal constant $c_{11} > 0$.  
Furthermore, since $\Px(x : h_{w_{k-1}}(x) \neq h_{w^{*}}(x)) \leq 2^{-k-3}$, 
we know that the angle between $w_{k-1}$ and $w^{*}$ is at most $2^{-k-3} \pi$,
so that 
\begin{multline*}
\|w_{k-1} - w^{*}\| 
= \sqrt{ 2 - 2 w_{k-1} \cdot w^{*} } 
\leq \sqrt{ 2 - 2 \cos(2^{-k-3} \pi) } 
\\ \leq \sqrt{ 2 - 2 \cos^{2}(2^{-k-3} \pi) } 
= \sqrt{2} \sin(2^{-k-3} \pi) \leq 2^{-k-3} \pi \sqrt{2}.
\end{multline*}
For $c_{10} = \pi\sqrt{2} 2^{-3}$, this is $r_{k}$.
By Hoeffding's inequality (under the conditional distribution given $|W_{k}|$), the law of total probability,
Lemma~\ref{lem:denoised-risk}, and linearity of conditional expectations, 
with probability at least $1-\delta_{k}/3$, for $X \sim \Px$,
\begin{multline}
\label{eqn:opt-loss-bound}
\sum_{(x,y) \in W_{k}} \ell_{\tau_{k}}( y ( w^{*} \cdot x) )
\leq |W_{k}| \E\left[ \ell_{\tau_{k}}(|w^{*} \cdot X|) \Big| w_{k-1}, |w_{k-1} \cdot X| \leq b_{k-1} \right] 
\\ + c_{1} |W_{k}| \sqrt{2^{k} \change M_{1}} \frac{z_{k}}{\tau_{k}} 
+ \sqrt{ |W_{k}| (1/2) c_{11}^{2} d \ln(3/\delta_{k}) }.
\end{multline}
We bound each term on the right hand side separately.
By Lemma~\ref{lem:opt-margin-loss}, the first term is at most $|W_{k}|\frac{c_{3} \tau_{k}}{c_{2} b_{k-1}} = |W_{k}|\frac{c_{3} c_{8}}{2 c_{2} c_{7}}$.
Next, 
\begin{equation*}
\frac{z_{k}}{\tau_{k}} 
= \frac{\sqrt{c_{10}^{2} 2^{-2k}/(d-1) + 4 c_{7}^{2} 2^{-2k}/d}}{c_{8} 2^{-k} / \sqrt{d}}
\leq \frac{\sqrt{ 2c_{10}^{2} + 4 c_{7}^{2}}}{c_{8}},
\end{equation*}
while $2^{k} \leq 2/\alpha$
so that the second term is at most 
\begin{equation*}
\sqrt{2} c_{1} \frac{\sqrt{ 2c_{10}^{2} + 4 c_{7}^{2}}}{c_{8}} |W_{k}| \sqrt{ \frac{\change m}{\alpha} }.
\end{equation*}
Noting that 
\begin{equation}
\label{eqn:m-bound}
M_{1} = \sum_{k^{\prime}=1}^{\lceil \log_{2}(1/\alpha) \rceil} m_{k^{\prime}} 
\leq \frac{32 c_{5}}{\kappa^{2}} \frac{1}{\alpha} d \log\left(\frac{1}{\kappa\delta}\right),
\end{equation}
we find that the second term on the right hand side of \eqref{eqn:opt-loss-bound} is at most 
\begin{equation*}
\sqrt{\frac{c_{5}}{c_{9}}} \frac{8 c_{1}}{\kappa} \frac{\sqrt{ 2c_{10}^{2} + 4 c_{7}^{2}}}{c_{8}} |W_{k}| \sqrt{ \frac{\change d \log\left(\frac{1}{\kappa\delta}\right)}{\alpha^{2}} }
= \frac{8 c_{1} \sqrt{c_{5}}}{\kappa} \frac{\sqrt{ 2c_{10}^{2} + 4 c_{7}^{2}}}{c_{8}c_{9}} |W_{k}|.
\end{equation*}
Finally, since $d \ln(3/\delta_{k}) \leq 2 d \ln(1/\delta_{k}) \leq \frac{2 \kappa^{2}}{c_{5}} 2^{-k} m_{k}$,
and \eqref{eqn:Wk-bounds} implies $2^{-k} m_{k} \leq \frac{1}{c_{2} c_{7}} |W_{k}|$, 
the third term on the right hand side of \eqref{eqn:opt-loss-bound} is at most
\begin{equation*}
|W_{k}| \frac{c_{11} \kappa}{ \sqrt{c_{2} c_{5} c_{7}} }.
\end{equation*}
Altogether, we have
\begin{equation*}
\sum_{(x,y) \in W_{k}} \ell_{\tau_{k}}( y ( w^{*} \cdot x) )
\leq |W_{k}| \left( 
\frac{c_{3} c_{8}}{2 c_{2} c_{7}} 
+  \frac{8 c_{1} \sqrt{c_{5}}}{\kappa} \frac{\sqrt{ 2c_{10}^{2} + 4 c_{7}^{2}}}{c_{8}c_{9}}
+ \frac{c_{11} \kappa}{ \sqrt{c_{2} c_{5} c_{7}} }\right).
\end{equation*}
Taking $c_{9} = 1/\kappa^{3}$ and $c_{8} = \kappa$, this is at most
\begin{equation*}
\kappa |W_{k}| \left( 
\frac{c_{3}}{2 c_{2} c_{7}} 
+  8 c_{1} \sqrt{c_{5}}\sqrt{ 2c_{10}^{2} + 4 c_{7}^{2}}
+ \frac{c_{11}}{ \sqrt{c_{2} c_{5} c_{7}} }\right).
\end{equation*}

Next, note that because $h_{w_{k}}(x) \neq y \Rightarrow \ell_{\tau_{k}}(y (v_{k} \cdot x)) \geq 1$, 
and because (as proven above) $\|w^{*} - w_{k-1}\| \leq r_{k}$,
\begin{equation*}
|W_{k}| \er_{W_{k}}( h_{w_{k}} ) 
\leq \sum_{(x,y) \in W_{k}} \ell_{\tau_{k}}(y (v_{k} \cdot x))
\leq \sum_{(x,y) \in W_{k}} \ell_{\tau_{k}}(y (w^{*} \cdot x)) + \kappa |W_{k}|.
\end{equation*}
Combined with the above, we have
\begin{equation*}
|W_{k}| \er_{W_{k}}( h_{w_{k}} ) 
\leq \kappa |W_{k}| \left( 
1 + \frac{c_{3}}{2 c_{2} c_{7}} 
+  8 c_{1} \sqrt{c_{5}}\sqrt{ 2c_{10}^{2} + 4 c_{7}^{2}}
+ \frac{c_{11}}{ \sqrt{c_{2} c_{5} c_{7}} }\right).
\end{equation*}
Let $c_{12} = 1 + \frac{c_{3}}{2 c_{2} c_{7}} +  8 c_{1} \sqrt{c_{5}}\sqrt{ 2c_{10}^{2} + 4 c_{7}^{2}} + \frac{c_{11}}{ \sqrt{c_{2} c_{5} c_{7}} }$.
Furthermore, 
\begin{multline*}
|W_{k}|\er_{W_{k}}( h_{w_{k}} ) 
= \sum_{(x,y) \in W_{k}} \ind[ h_{w_{k}}(x) \neq y ]
\\ \geq \sum_{(x,y) \in W_{k}} \ind[ h_{w_{k}}(x) \neq h_{w^{*}}(x) ]  - \sum_{(x,y) \in W_{k}} \ind[ h_{w^{*}}(x) \neq y ].
\end{multline*}
For an appropriately large value of $c_{5}$, 
by a Chernoff bound, with probability at least $1-\delta_{k}/3$, 
\begin{equation*}
\sum_{s=t+\sum_{j=0}^{k-1}m_{j} + 1}^{t+\sum_{j=0}^{k} m_{j}} \ind[ h_{w^{*}}(X_{s}) \neq Y_{s} ] 
\leq 2 e \change M_{1} m_{k} + \log_{2}(3/\delta_{k}).
\end{equation*}
In particular, this implies 
\begin{equation*}
\sum_{(x,y) \in W_{k}} \ind[ h_{w^{*}}(x) \neq y ]
\leq 2 e \change M_{1} m_{k} + \log_{2}(3/\delta_{k}),
\end{equation*}
so that 
\begin{equation*}
\sum_{(x,y) \in W_{k}} \ind[ h_{w_{k}}(x) \neq h_{w^{*}}(x) ]
\leq |W_{k}|\er_{W_{k}}( h_{w_{k}} ) + 2 e \change M_{1} m_{k} + \log_{2}(3/\delta_{k}).
\end{equation*}
Noting that \eqref{eqn:m-bound} and \eqref{eqn:Wk-bounds} imply 
\begin{align*}
\change M_{1} m_{k} & \leq \change \frac{32 c_{5}}{\kappa^{2}} \frac{ d \log\left(\frac{1}{\kappa\delta}\right) }{c_{9} \sqrt{ \change d \log\left(\frac{1}{\kappa\delta}\right)}} \frac{2^{k}}{c_{2} c_{7}} |W_{k}|
\leq \frac{32 c_{5}}{c_{2} c_{7} c_{9} \kappa^{2}} \sqrt{ \change d \log\left(\frac{1}{\kappa\delta}\right) } 2^{k} |W_{k}|
\\ & = \frac{32 c_{5}}{c_{2} c_{7} c_{9}^{2} \kappa^{2}} \alpha 2^{k} |W_{k}|
= \frac{32 c_{5} \kappa^{4}}{c_{2} c_{7}} \alpha 2^{k} |W_{k}|
\leq \frac{32 c_{5} \kappa^{4}}{c_{2} c_{7}} |W_{k}|,
\end{align*}
and \eqref{eqn:Wk-bounds} implies $\log_{2}(3/\delta_{k}) \leq \frac{2\kappa^{2}}{c_{2}c_{5}c_{7}}|W_{k}|$,
altogether we have
\begin{align*}
\sum_{(x,y) \in W_{k}} \ind[ h_{w_{k}}(x) \neq h_{w^{*}}(x) ]
& \leq |W_{k}|\er_{W_{k}}( h_{w_{k}} ) + \frac{64 e c_{5} \kappa^{4}}{c_{2} c_{7}} |W_{k}| + \frac{2\kappa^{2}}{c_{2}c_{5}c_{7}}|W_{k}|
\\ & \leq \kappa |W_{k}| \left( c_{12} + \frac{64 e c_{5} \kappa^{3}}{c_{2} c_{7}} + \frac{2\kappa}{c_{2}c_{5}c_{7}} \right).
\end{align*}
Letting $c_{13} = c_{12} + \frac{64 e c_{5}}{c_{2} c_{7}} + \frac{2}{c_{2}c_{5}c_{7}}$, and noting $\kappa \leq 1$, 
we have
$\sum_{(x,y) \in W_{k}} \ind[ h_{w_{k}}(x) \neq h_{w^{*}}(x) ] \leq c_{13} \kappa |W_{k}|$.

Lemma~\ref{lem:vc-ratio} (applied under the conditional distribution given $|W_{k}|$)
and the law of total probability imply that with probability at least $1-\delta_{k}/3$,
\begin{align*}
|W_{k}| &\Px\left( x : h_{w_{k}}(x) \neq h_{w^{*}}(x) \Big| |w_{k-1} \cdot x| \leq b_{k-1}\right)
\\ & \leq \sum_{(x,y) \in W_{k}} \ind[ h_{w_{k}}(x) \neq h_{w^{*}}(x)]
+ c_{14} \sqrt{ |W_{k}| (d \log(|W_{k}|/d) + \log(1/\delta_{k})) },
\end{align*}
for a universal constant $c_{14} > 0$.
Combined with the above, and the fact that \eqref{eqn:Wk-bounds} implies 
$\log(1/\delta_{k}) \leq \frac{\kappa^{2}}{c_{2}c_{5}c_{7}}|W_{k}|$
and 
\begin{align*}
d \log(|W_{k}|/d) & \leq d \log\left(\frac{8c_{3}c_{5}c_{7} \log\left(\frac{1}{\kappa\delta_{k}}\right)}{\kappa^{2}}\right)
\\ & \leq d \log\left(\frac{8 c_{3} c_{5} c_{7}}{\kappa^{3} \delta_{k}}\right) 
\leq 3\log(8 \max\{c_{3},1\} c_{5} ) c_{5} d \log\left(\frac{1}{\kappa \delta_{k}}\right)
\\ & \leq 3 \log(8 \max\{c_{3},1\}) \kappa^{2} 2^{-k} m_{k} 
\leq \frac{3 \log(8 \max\{c_{3},1\})}{c_{2} c_{7}} \kappa^{2} |W_{k}|,
\end{align*}
we have 
\begin{align*}
|W_{k}| & \Px\left( x : h_{w_{k}}(x) \neq h_{w^{*}}(x) \Big| |w_{k-1} \cdot x| \leq b_{k-1}\right)
\\ & \leq c_{13} \kappa |W_{k}|
+ c_{14} \sqrt{ |W_{k}| \left( \frac{3 \log(8 \max\{c_{3},1\})}{c_{2} c_{7}} \kappa^{2} |W_{k}| + \frac{\kappa^{2}}{c_{2}c_{5}c_{7}}|W_{k}| \right)}
\\ & = \kappa |W_{k}| \left( c_{13} + c_{14} \sqrt{ \frac{3 \log(8 \max\{c_{3},1\})}{c_{2} c_{7}} + \frac{1}{c_{2}c_{5}c_{7}}}\right).
\end{align*}
Thus, letting $c_{15} = \left( c_{13} + c_{14} \sqrt{ \frac{3 \log(8 \max\{c_{3},1\})}{c_{2} c_{7}} + \frac{1}{c_{2}c_{5}c_{7}}}\right)$,
we have
\begin{equation}
\label{eqn:conditional-error-bound}
\Px\left( x : h_{w_{k}}(x) \neq h_{w^{*}}(x) \Big| |w_{k-1} \cdot x| \leq b_{k-1}\right)
\leq c_{15} \kappa.
\end{equation}

Next, note that $\|v_{k} - w_{k-1}\|^{2} = \|v_{k}\|^{2} + 1 - 2 \|v_{k}\| \cos( \pi \Px(x : h_{w_{k}}(x) \neq h_{w_{k-1}}(x)) )$.
Thus, one implication of the fact that $\|v_{k} - w_{k-1}\| \leq r_{k}$ is that 
$\frac{\|v_{k}\|}{2} + \frac{1-r_{k}^{2}}{2\|v_{k}\|} \leq \cos( \pi \Px(x : h_{w_{k}}(x) \neq h_{w_{k-1}}(x)) )$;
since the left hand side is positive, we have $\Px(x : h_{w_{k}}(x) \neq h_{w_{k-1}}(x)) < 1/2$. 
Additionally, by differentiating, one can easily verify that for $\phi \in [0,\pi]$, 
$x \mapsto \sqrt{ x^{2} + 1 - 2 x \cos(\phi) }$ is minimized at $x=\cos(\phi)$, 
in which case $\sqrt{x^{2} + 1 - 2 x \cos(\phi) } = \sin(\phi)$.
Thus, $\|v_{k} - w_{k-1}\| \geq \sin( \pi \Px(x : h_{w_{k}}(x) \neq h_{w_{k-1}}(x) ) )$.
Since $\|v_{k} - w_{k-1}\| \leq r_{k}$, 
we have $\sin(\pi \Px(x : h_{w_{k}}(x) \neq h_{w_{k-1}}(x))) \leq r_{k}$.
Since $\sin(\pi x) \geq x$ for all $x \in [0,1/2]$, 
combining this with the fact (proven above) that $\Px(x : h_{w_{k}}(x) \neq h_{w_{k-1}}(x)) < 1/2$
implies $\Px(x : h_{w_{k}}(x) \neq h_{w_{k-1}}(x)) \leq r_{k}$.

In particular, we have that both $\Px(x : h_{w_{k}}(x) \neq h_{w_{k-1}}(x)) \leq r_{k}$ and $\Px(x : h_{w^{*}}(x) \neq h_{w_{k-1}}(x)) \leq 2^{-k-3} \leq r_{k}$.
Now Lemma~\ref{lem:margin-error-concentration} implies that, for any universal constant $c > 0$, 
there exists a corresponding universal constant $c^{\prime} > 0$ such that
\begin{equation*}
\Px\left(x : h_{w_{k}}(x) \neq h_{w_{k-1}}(x) \text{ and } |w_{k-1} \cdot x| \geq c^{\prime} \frac{r_{k}}{\sqrt{d}} \right) \leq c r_{k}
\end{equation*}
and
\begin{equation*}
\Px\left(x : h_{w^{*}}(x) \neq h_{w_{k-1}}(x) \text{ and } |w_{k-1} \cdot x| \geq c^{\prime} \frac{r_{k}}{\sqrt{d}} \right) \leq c r_{k},
\end{equation*}
so that (by a union bound)
\begin{align*}
& \Px\left(x : h_{w_{k}}(x) \neq h_{w^{*}}(x) \text{ and } |w_{k-1} \cdot x| \geq c^{\prime} \frac{r_{k}}{\sqrt{d}} \right) 
\\ & \leq 
\Px\left(x : h_{w_{k}}(x) \neq h_{w_{k-1}}(x) \text{ and } |w_{k-1} \cdot x| \geq c^{\prime} \frac{r_{k}}{\sqrt{d}} \right) 
\\ & +
\Px\left(x : h_{w^{*}}(x) \neq h_{w_{k-1}}(x) \text{ and } |w_{k-1} \cdot x| \geq c^{\prime} \frac{r_{k}}{\sqrt{d}} \right) 
\leq 2 c r_{k}.
\end{align*}
In particular, letting $c_{7} = c^{\prime} c_{10} / 2$, we have $c^{\prime} \frac{r_{k}}{\sqrt{d}} = b_{k-1}$.
Combining this with \eqref{eqn:conditional-error-bound}, Lemma~\ref{lem:uniform-P-concentration}, and a union bound, we have that
\begin{align*}
& \Px\left( x : h_{w_{k}}(x) \neq h_{w^{*}}(x)\right)
\\ & \leq \Px\left(x : h_{w_{k}}(x) \neq h_{w^{*}}(x) \text{ and } |w_{k-1} \cdot x| \geq b_{k-1} \right) 
\\ & {\hskip 3mm}+ \Px\left(x : h_{w_{k}}(x) \neq h_{w^{*}}(x) \text{ and } |w_{k-1} \cdot x| \leq b_{k-1} \right) 
\\ & \leq 2 c r_{k} + \Px\left( x : h_{w_{k}}(x) \neq h_{w^{*}}(x) \Big| |w_{k-1} \cdot x| \leq b_{k-1} \right) \Px\left(x : |w_{k-1} \cdot x| \leq b_{k-1}\right)
\\ & \leq 2 c r_{k} + c_{15} \kappa c_{3} b_{k-1} \sqrt{d}
= \left( 2^{5} c c_{10} + c_{15} \kappa c_{3} c_{7} 2^{5} \right) 2^{-k-4}.
\end{align*}
Taking $c = \frac{1}{2^{6} c_{10}}$ and $\kappa = \frac{1}{2^{6} c_{3} c_{7} c_{15}}$,
we have $\Px(x : h_{w_{k}}(x) \neq h_{w^{*}}(x)) \leq 2^{-k-4}$, as required.

By a union bound, this occurs with probability at least $1 - (4/3)\delta_{k}$.
\qed
\end{proof}

\begin{proof}[Proof of Theorem~\ref{thm:linsep-uniform}]
We begin with the bound on the error rate.
If $\change > \frac{\pi^{2}}{400 \cdot 2^{27} (d+\ln(4/\delta))}$, the result trivially holds, since then $1 \leq \frac{400 \cdot 2^{27}}{\pi^{2}} \sqrt{\change (d+\ln(4/\delta))}$.
Otherwise, suppose $\change \leq \frac{\pi^{2}}{400 \cdot 2^{27} (d+\ln(4/\delta))}$.

Fix any $i \in \nats$.
Lemma~\ref{lem:perceptron-init} implies that, with probability at least $1-\delta/4$, 
the $w_{0}$ returned in Step 0 of ${\rm ABL}(M(i-1),\tilde{h}_{i-1})$ satisfies
$\Px(x : h_{w_{0}}(x) \neq h_{M(i-1) + m_{0}+1}^{*}(x)) \leq 1/16$.  
Taking this as a base case, Lemma~\ref{lem:margin-error-bound} then inductively implies that, 
with probability at least 
\begin{multline*}
1 - \frac{\delta}{4} - \sum_{k=1}^{\lceil \log_{2}(1/\alpha) \rceil} (4/3) \frac{\delta}{2(\lceil \log_{2}(4/\alpha) \rceil - k)^{2}}
\geq 1 - \frac{\delta}{2} \left(1 + (4/3) \sum_{\ell=2}^{\infty} \frac{1}{\ell^{2}} \right) 
\geq 1 - \delta,
\end{multline*}
every $k \in \{ 0, 1, \ldots, \lceil \log_{2}(1/\alpha) \rceil \}$ has
\begin{equation}
\label{eqn:abl-mistake-prob-raw}
\Px(x : h_{w_{k}}(x) \neq h_{M(i-1)+m_{0}+1}^{*}(x)) \leq 2^{-k-4},
\end{equation}
and furthermore the number of labels requested during ${\rm ABL}(M(i-1),\tilde{h}_{i-1})$ total to at most (for appropriate universal constants $\hat{c}_{1},\hat{c}_{2}$)
\begin{align*}
m_{0} + \!\!\!\!\sum_{k=1}^{\lceil \log_{2}(1/\alpha) \rceil} |W_{k}|
& \leq \hat{c}_{1} \left(d + \ln\left(\frac{1}{\delta}\right) + \sum_{k=1}^{\lceil \log_{2}(1/\alpha) \rceil}  d \log\left(\frac{( \lceil \log_{2}(4/\alpha) \rceil - k )^{2}}{\delta}\right) \right)
\\ & \leq \hat{c}_{2} d \log\left(\frac{1}{\change d}\right)\log\left(\frac{1}{\delta}\right).
\end{align*}
In particular, by a union bound, \eqref{eqn:abl-mistake-prob-raw} implies that for every $k \in \{1,\ldots,\lceil \log_{2}(1/\alpha) \rceil\}$,
every 
\begin{equation*}
m \in \left\{ M(i-1) + \sum_{j=0}^{k-1} m_{j} + 1, \ldots, M(i-1) + \sum_{j=0}^{k} m_{j} \right\}
\end{equation*}
has
\begin{align*}
& \Px(x : h_{w_{k-1}}(x) \neq h_{m}^{*}(x)) 
\\ & \leq \Px(x : h_{w_{k-1}}(x) \neq h_{M(i-1)+m_{0}+1}^{*}(x)) + \Px(x : h_{M(i-1)+m_{0}+1}^{*}(x) \neq h_{m}^{*}(x)) 
\\ & \leq 2^{-k-3} + \change M.
\end{align*}
Thus, noting that
\begin{align*}
M & = \sum_{k=0}^{\lceil \log_{2}(1/\alpha) \rceil} m_{k}
= \Theta\left( d + \log\left(\frac{1}{\delta}\right) + \sum_{k=1}^{\lceil \log_{2}(1/\alpha) \rceil} 2^{k} d \log\left(\frac{\lceil \log_{2}(1/\alpha) \rceil - k}{\delta}\right) \right)
\\ & = \Theta\left( \frac{1}{\alpha} d \log\left(\frac{1}{\delta}\right) \right)
= \Theta\left(\sqrt{\frac{d}{\change} \log\left(\frac{1}{\delta}\right)} \right),
\end{align*}
with probability at least $1-\delta$,
\begin{equation*}
\Px(x : h_{w_{\lceil \log_{2}(1/\alpha) \rceil-1}}(x) \neq \target_{M i}(x) ) \leq O\left( \alpha + \change M \right) = O\left( \sqrt{ \change d \log\left(\frac{1}{\delta}\right) } \right).
\end{equation*} 
In particular, this implies that, with probability at least $1-\delta$, every $t \in \{M i + 1, \ldots, M (i+1)-1\}$ has
\begin{align*}
\er_{t}(\hat{h}_{t}) & \leq \Px(x : h_{w_{\lceil \log_{2}(1/\alpha) \rceil-1}}(x) \neq \target_{M i}(x) ) + \Px( x : \target_{M i}(x) \neq \target_{t}(x) )
\\ & \leq O\left( \sqrt{ \change d \log\left(\frac{1}{\delta}\right) } \right) + \change M
= O\left( \sqrt{ \change d \log\left(\frac{1}{\delta}\right) } \right),
\end{align*}
which completes the proof of the bound on the error rate.

Setting $\delta = \sqrt{\change d}$, and noting that $\ind[ \hat{Y}_{t} \neq Y_{t} ] \leq 1$, we have that for any $t > M$, 
\begin{equation*}
\P\left( \hat{Y}_{t} \neq Y_{t} \right) 
= \E\left[ \er_{t}(\hat{h}_{t}) \right] 
\leq O\left( \sqrt{ \change d \log\left(\frac{1}{\delta}\right) } \right) + \delta 
= O\left( \sqrt{ \change d \log\left(\frac{1}{\change d}\right) } \right).
\end{equation*}
Thus, by linearity of the expectation, 
\begin{equation*}
\E\left[ \sum_{t=1}^{T} \ind\left[ \hat{Y}_{t} \neq Y_{t} \right] \right]
\leq M + O\left( \sqrt{ \change d \log\left(\frac{1}{\change d}\right) } T \right)
= O\left( \sqrt{ \change d \log\left(\frac{1}{\change d}\right) } T \right).
\end{equation*}
Furthermore, as mentioned, with probability at least $1-\delta$, 
the number of labels requested during the execution of ${\rm ABL}(M(i-1),\tilde{h}_{i-1})$ is at most
\begin{equation*}
O\left( d \log\left(\frac{1}{\change d}\right)\log\left(\frac{1}{\delta}\right) \right).
\end{equation*}
Thus, since the number of labels requested during the execution of ${\rm ABL}(M(i-1),\tilde{h}_{i-1})$ cannot exceed $M$, 
letting $\delta = \sqrt{\change d}$, the expected number of requested labels during this execution is at most
\begin{align*}
O\left( d \log^{2}\left(\frac{1}{\change d}\right) \right) + \sqrt{\change d} M
& \leq O\left( d \log^{2}\left(\frac{1}{\change d}\right) \right) + O\left( d \sqrt{\log\left(\frac{1}{\change d}\right) } \right)
\\ & = O\left( d \log^{2}\left(\frac{1}{\change d}\right) \right).
\end{align*}
Thus, by linearity of the expectation, the expected number of labels requested among the first $T$ samples is at most
\begin{equation*}
O\left( d \log^{2}\left(\frac{1}{\change d}\right) \left\lceil \frac{T}{M} \right\rceil \right)
= O\left( \sqrt{\change d} \log^{3/2}\left(\frac{1}{\change d}\right) T \right),
\end{equation*}
which completes the proof.
\qed
\end{proof}

\paragraph{Remark:} The original work of \citet{min_concept} additionally allowed for some number $K$ of ``jumps'':
times $t$ at which $\change_{t} = 1$.  Note that, in the above algorithm, since the influence of each sample is localized to the predictors trained
within that ``batch'' of $M$ instances, the effect of allowing such jumps would only change the bound on the number of
mistakes to $\tilde{O}\left(\sqrt{d \change} T + \sqrt{\frac{d}{\change}} K \right)$.  This compares favorably to the 
result of \citet{min_concept}, which is roughly $O\left( (d \change)^{1/4} T + \frac{d^{1/4}}{\change^{3/4}} K \right)$.
However, the result of \citet{min_concept} was proven for a more general setting, allowing distributions $\Px$
that are not uniform (though they do require a relation between the angle between any two separators and the 
probability mass they disagree on, similar to that holding for the uniform distribution, which seems to require that the 
distributions approximately retain some properties of the uniform distribution).  It is not clear whether Theorem~\ref{thm:linsep-uniform} can be
generalized to this larger family of distributions.

%% file: tex-files/general-active.tex
As mentioned, the above results on linear separators also provide results 
for the number of queries in \emph{active learning}.  One can also state
quite general results on the expected number of queries and mistakes
achievable by an active learning algorithm.  
This section provides such results, for an algorithm based on the 
the well-known strategy of \emph{disagreement-based} active learning.
Throughout this section, we suppose $\targetseq \in S_{\change}$,
for a given $\change \in (0,1]$: that is, $\Px( x : \target_{t+1}(x) \neq \target_{t}(x)) \leq \change$
for all $t \in \nats$.

First, we introduce a few definitions.
For any set $\H \subseteq \C$, define the \emph{region of disagreement}
\begin{equation*}
\DIS(\H) = \{x \in \X : \exists h,g \in \H \text{ s.t. } h(x) \neq g(x) \}.
\end{equation*}
The analysis in this section is centered around the following algorithm.
The ${\rm Active}$ subroutine is from the work of \citet{hanneke:activized} (slightly modified here),
and is a variant of the $A^2$ (Agnostic Acive) algorithm of \citet{A2};
the appropriate values of $M$ and $\hat{T}_{k}(\cdot)$ will be discussed below.

\begin{bigboxit}
Algorithm: ${\rm DriftingActive}$\\
0. For $i = 1,2,\ldots$\\
1. \quad ${\rm Active}(M (i-1))$\\
\end{bigboxit}
\begin{bigboxit}
Subroutine: ${\rm Active}(t)$\\
0. Let $\hat{h}_{0}$ be an arbitrary element of $\C$, and let $V_{0} \gets \C$\\
1. Predict $\hat{Y}_{t+1} = \hat{h}_{0}(X_{t+1})$ as the prediction for the value of $Y_{t+1}$\\
2. For $k = 0,1,\ldots,\log_{2}(M/2)$\\
3. \quad $Q_{k} \gets \{\}$\\
4. \quad For $s = 2^{k}+1,\ldots,2^{k+1}$\\
5. \qquad Predict $\hat{Y}_{s} = \hat{h}_{k}(X_{s})$ as the prediction for the value of $Y_{s}$\\
6. \qquad If $X_{s} \in \DIS(V_{k})$\\
7. \quad\qquad Request the label $Y_{s}$ and let $Q_{k} \gets Q_{k} \cup \{(X_{s},Y_{s})\}$\\
8. \quad Let $\hat{h}_{k+1} = \argmin_{h \in V_{k}} \sum_{(x,y) \in Q_{k}} \ind[h(x) \neq y]$\\
9. \quad Let $V_{k+1} \gets \{h \in V_{k} : \sum_{(x,y) \in Q_{k}} \ind[h(x) \neq y] - \ind[\hat{h}_{k+1}(x) \neq y] \leq \hat{T}_{k}\}$
\end{bigboxit}

As in the ${\rm DriftingHalfspaces}$ algorithm above, this ${\rm DriftingActive}$
algorithm proceeds in batches, and in each batch runs an active learning algorithm
designed to be robust to classification noise.  This robustness to classification noise
translates into our setting as tolerance for the fact that there is no classifier in $\C$
that perfectly classifies all of the data.  The specific algorithm employed here maintains
a set $V_{k} \subseteq \C$ of candidate classifiers, and requests the labels of samples $X_{s}$
for which there is some disagreement on the classification among classifiers in $V_{k}$.
We maintain the invariant that there is a low-error classifier contained in $V_{k}$ at all
times, and thus the points we query provide some information to help us determine
which among these remaining candidates has low error rate.  Based on these queries,
we periodically (in Step 9) remove from $V_{k}$ those classifiers making a relatively excessive 
number of mistakes on the queried samples, relative to the minimum among classifiers in $V_{k}$.
All predictions are made with an element of $V_{k}$.\footnote{One could alternatively proceed
as in ${\rm DriftingHalfspaces}$, using the final classifier from the previous batch, which 
would also add a guarantee on the error rate achieved at all sufficiently large $t$.}

We prove an abstract bound on the number of labels requested by this algorithm,
expressed in terms of the \emph{disagreement coefficient} \citep{hanneke:07b},
defined as follows.  For any $r \geq 0$ and any classifier $h$, define $\Ball(h,r) = \{g \in \C : \Px(x : g(x) \neq h(x)) \leq r\}$.
Then for $r_{0} \geq 0$ and any classifier $h$, define the disagreement coefficient of $h$ with respect to $\C$ under $\Px$:
\begin{equation*}
\dc_{h}(r_{0}) = \sup_{r > r_{0}} \frac{ \Px( \DIS( \Ball( h, r ) ) ) }{r}.
\end{equation*}
Usually, the disagreement coefficient would be used with $h$ equal the target concept;
however, since the target concept is not fixed in our setting, 
we will make use of the worst-case value of the disagreement coefficient:
$\dc_{\C}(r_{0}) = \sup_{h \in \C} \dc_{h}(r_{0})$.
This quantity has been bounded for a variety of spaces $\C$ and distributions $\Px$
(see e.g., \cite{hanneke:07b,el-yaniv:12,balcan:13}).
It is useful in bounding how quickly the region $\DIS(V_{k})$ collapses in the 
algorithm.  Thus, since the probability the algorithm requests the label of the next instance
is $\Px(\DIS(V_{k}))$, the quantity $\dc_{\C}(r_{0})$ naturally arises in characterizing the
number of labels we expect this algorithm to request.
Specifically, we have the following result.\footnote{Here,
we define $\lceil x \rceil_{2} = 2^{\lceil \log_{2}(x) \rceil}$, for $x \geq 1$.}

\begin{theorem}
\label{thm:general-active}
For an appropriate universal constant $c_{1} \in [1,\infty)$, 
if $\targetseq \in S_{\change}$ for some $\change \in (0,1]$,
then taking $M = \left\lceil c_{1} \sqrt{\frac{\vc}{\change}} \right\rceil_{2}$,
and $\hat{T}_{k} = \log_{2}(1/\sqrt{\vc \change}) + 2^{2k+2} e \change$,
and defining $\epsilon_{\change} = \sqrt{\vc\change} \Log(1/(\vc\change))$,
the above ${\rm DriftingActive}$ algorithm makes an expected number of mistakes among the 
first $T$ instances that is 
\begin{equation*}
O\left(\epsilon_{\change} \Log(\vc/\change) T \right) = \tilde{O}\left( \sqrt{\vc\change} \right) T
\end{equation*}
and requests an expected number of labels among the first $T$ instances that is 
\begin{equation*}
O\left( \dc_{\C}( \epsilon_{\change} ) \epsilon_{\change} \Log(\vc/\change) T \right) = \tilde{O}\left( \dc_{\C}(\sqrt{\vc \change}) \sqrt{\vc \change} \right) T.
\end{equation*}
\end{theorem}

The proof of Theorem~\ref{thm:general-active} relies on an analysis of the behavior of the ${\rm Active}$ subroutine,
characterized in the following lemma.

\begin{lemma}
\label{lem:active-subroutine}
Fix any $t \in \nats$, and consider the values obtained in the execution of ${\rm Active}(t)$.
Under the conditions of Theorem~\ref{thm:general-active},
there is a universal constant $c_{2} \in [1,\infty)$ such that,
for any $k \in \{0,1,\ldots,\log_{2}(M/2)\}$,
with probability at least $1-2\sqrt{\vc \change}$, if 
$\target_{t+1} \in V_{k}$,
then $\target_{t+1} \in V_{k+1}$ and
$\sup_{h \in V_{k+1}} \Px(x : h(x) \neq \target_{t+1}(x)) \leq c_{2} 
2^{-k} \vc \Log(c_{1} / \sqrt{\vc\change})$.
\end{lemma}
\begin{proof}
By a Chernoff bound, with probability at least $1-\sqrt{\vc \change}$,
\begin{equation*}
\sum_{s=2^{k}+1}^{2^{k+1}} \ind[\target_{t+1}(X_{s}) \neq Y_{s}]
\leq \log_{2}(1/\sqrt{\vc \change}) + 2^{2k+2} e \change
= \hat{T}_{k}.
\end{equation*} 
Therefore, if $\target_{t+1} \in V_{k}$, then since every $g \in V_{k}$
agrees with $\target_{t+1}$ on those points $X_{s} \notin \DIS(V_{k})$, 
in the update in Step 9 defining $V_{k+1}$,
we have 
\begin{align*}
& \sum_{(x,y) \in Q_{k}} \ind[\target_{t+1}(x) \neq y] - \ind[\hat{h}_{k+1}(x) \neq y]
\\ & = \sum_{s=2^{k}+1}^{2^{k+1}} \ind[\target_{t+1}(X_{s}) \neq Y_{s}]
- \min_{g \in V_{k}} \sum_{s=2^{k}+1}^{2^{k+1}} \ind[g(X_{s}) \neq Y_{s}]
\\ & \leq \sum_{s=2^{k}+1}^{2^{k+1}} \ind[\target_{t+1}(X_{s}) \neq Y_{s}] \leq \hat{T}_{k},
\end{align*}
so that $\target_{t+1} \in V_{k+1}$ as well.

Furthermore, if $\target_{t+1} \in V_{k}$, 
then by the definition of $V_{k+1}$, 
we know every $h \in V_{k+1}$ has
\begin{equation*}
\sum_{s=2^{k}+1}^{2^{k+1}} \ind[ h(X_{s}) \neq Y_{s} ] 
\leq \hat{T}_{k} + \sum_{s=2^{k}+1}^{2^{k+1}} \ind[ \target_{t+1}(X_{s}) \neq Y_{s} ], 
\end{equation*}
so that a triangle inequality implies
\begin{align*}
\sum_{s=2^{k}+1}^{2^{k+1}} \ind[ h(X_{s}) \neq \target_{t+1}(X_{s}) ]
& \leq 
\sum_{s=2^{k}+1}^{2^{k+1}} \ind[ h(X_{s}) \neq Y_{s} ]
+ \ind[ \target_{t+1}(X_{s}) \neq Y_{s} ]
\\ & \leq 
\hat{T}_{k} + 2 \sum_{s=2^{k}+1}^{2^{k+1}} \ind[ \target_{t+1}(X_{s}) \neq Y_{s} ]
\leq 3 \hat{T}_{k}.
\end{align*}
Lemma~\ref{lem:vc-ratio} then implies that, on an additional event of 
probability at least $1-\sqrt{\vc \change}$,
every $h \in V_{k+1}$ has
\begin{align*}
& \Px(x : h(x) \neq \target_{t+1}(x))
\\ & \leq 2^{-k} 3\hat{T}_{k} + c 2^{-k} \sqrt{3\hat{T}_{k} (\vc \Log(2^{k}/\vc)+\Log(1/\sqrt{\vc\change}))} 
\\ & \phantom{\leq } + c 2^{-k} (\vc \Log(2^{k}/\vc) + \Log(1/\sqrt{\vc\change}))
\\ & \leq
2^{-k} 3 \log_{2}(1/\sqrt{\vc\change})
+ 2^{k} 12 e \change
+ c 2^{-k} \sqrt{ 6 \log_{2}(1/\sqrt{\vc\change}) \vc \Log(c_{1} / \sqrt{\vc\change})}
\\ & \phantom{\leq } + c 2^{-k} \sqrt{ 2^{2k} 24 e \change \vc \Log(c_{1} / \sqrt{\vc\change}) }
+ 2 c 2^{-k} \vc \Log(c_{1} / \sqrt{\vc\change})
\\ & 
\leq 
2^{-k} 3 \log_{2}(1/\sqrt{\vc\change})
+ 12 e c_{1} \sqrt{\vc\change}
+ 3 c 2^{-k} \sqrt{ \vc } \Log(c_{1} / \sqrt{\vc\change})
\\ & \phantom{\leq } + \sqrt{24 e} c \sqrt{\vc \change \Log(c_{1} / \sqrt{\vc\change}) }
+ 2 c 2^{-k} \vc \Log(c_{1} / \sqrt{\vc\change}),
\end{align*}
where $c$ is as in Lemma~\ref{lem:vc-ratio}.
Since $\sqrt{\vc \change} \leq 2 c_{1} \vc / M \leq c_{1} \vc 2^{-k}$,
this is at most
\begin{equation*}
\left(5 + 12 e c_{1}^{2} + 3 c + \sqrt{24 e} c c_{1} + 2 c\right)
2^{-k} \vc \Log(c_{1} / \sqrt{\vc\change}).
\end{equation*}
Letting $c_{2} = 5 + 12 e c_{1}^{2} + 3 c + \sqrt{24 e} c c_{1} + 2 c$, 
we have the result by a union bound.
\qed
\end{proof}

We are now ready for the proof of Theorem~\ref{thm:general-active}.

\begin{proof}[Proof of Theorem~\ref{thm:general-active}]
Fix any $i \in \nats$, and consider running ${\rm Active}(M(i-1))$.
Since $\target_{M(i-1)+1} \in \C$,
by Lemma~\ref{lem:active-subroutine}, a union bound, and induction, 
with probability at least $1-2\sqrt{\vc\change} \log_{2}(M/2)
\geq 1 - 2 \sqrt{\vc\change} \log_{2}(c_{1}\sqrt{\vc/\change})$,
every $k \in \{0,1,\ldots,\log_{2}(M/2)\}$ has 
\begin{equation}
\label{eqn:general-active-radius}
\sup_{h \in V_{k}} \Px(x : h(x) \neq \target_{M(i-1)+1}(x)) \leq 
c_{2} 2^{1-k} \vc \Log(c_{1} / \sqrt{\vc\change}).
\end{equation}
Thus, since $\hat{h}_{k} \in V_{k}$ for each $k$, 
the expected number of mistakes among the predictions 
$\hat{Y}_{M(i-1)+1},\ldots,\hat{Y}_{M i}$
is 

\begin{align*}
& 1 + \sum_{k=0}^{\log_{2}(M/2)} \sum_{s=2^{k}+1}^{2^{k+1}} \P(\hat{h}_{k}(X_{M(i-1)+s}) \neq Y_{M(i-1)+s})
\\ & \leq 1 + \sum_{k=0}^{\log_{2}(M/2)} \sum_{s=2^{k}+1}^{2^{k+1}}
\P(\target_{M(i-1)+1}(X_{M(i-1)+s}) \neq Y_{M(i-1)+s}) 
\\ & \phantom{\leq } + \sum_{k=0}^{\log_{2}(M/2)} \sum_{s=2^{k}+1}^{2^{k+1}} \P(\hat{h}_{k}(X_{M(i-1)+s}) \neq \target_{M(i-1)+1}(X_{M(i-1)+s}))
\\ & \leq 
1 + \change M^{2} +
\sum_{k=0}^{\log_{2}(M/2)} 2^{k} \left( c_{2} 2^{1-k} \vc \Log(c_{1} / \sqrt{\vc\change}) + 2\sqrt{\vc\change}\log_{2}(M/2)\right)
\\ & \leq
1 + 4 c_{1}^{2} \vc + 2 c_{2} \vc \Log(c_{1} / \sqrt{\vc\change}) \log_{2}(2 c_{1} \sqrt{\vc/\change})
+ 4c_{1} \vc \log_{2}(c_{1} \sqrt{\vc/\change})
\\ & = 
O\left( \vc \Log(\vc/\change) \Log(1/(\vc\change)) \right).
\end{align*}
Furthermore, \eqref{eqn:general-active-radius} implies the algorithm only 
requests the label $Y_{M(i-1)+s}$ for $s \in \{2^{k}+1,\ldots,2^{k+1}\}$
if $X_{M(i-1)+s} \in \DIS(\Ball(\target_{M(i-1)+1}, c_{2} 2^{1-k} \vc \Log(c_{1} / \sqrt{\vc\change})))$,
so that the expected number of labels requested among $Y_{M(i-1)+1},\ldots,Y_{M i}$ is at most
\begin{align*}
& 1 + \sum_{k=0}^{\log_{2}(M/2)} 2^{k} \left(\E[ \Px(\DIS(\Ball(\target_{M(i-1)+1}, c_{2} 2^{1-k} \vc \Log(c_{1}/\sqrt{\vc\change}))))] \right. 
\\ & {\hskip 6cm}\left.+ 2 \sqrt{\vc\change} \log_{2}(c_{1}\sqrt{\vc/\change})\right)
\\ & \leq 
1 + \dc_{\C}\left(4 c_{2} \vc \Log(c_{1}/\sqrt{\vc\change}) / M\right) 2 c_{2} \vc \Log(c_{2}/\sqrt{\vc\change}) \log_{2}(2 c_{1} \sqrt{\vc/\change})
\\ & {\hskip 6cm}+  4 c_{1} \vc \log_{2}(c_{1}\sqrt{\vc/\change})
\\ & = 
O\left( \dc_{\C}\left( \sqrt{\vc\change} \Log(1/(\vc\change)) \right) \vc \Log(\vc/\change) \Log(1/(\vc\change)) \right).
\end{align*}
Thus, the expected number of mistakes among indices $1,\ldots,T$ is at most
\begin{equation*}
O\left( \vc \Log(\vc/\change) \Log(1/(\vc\change)) \left\lceil \frac{T}{M} \right\rceil \right)
= O\left( \sqrt{\vc\change} \Log(\vc/\change) \Log(1/(\vc\change)) T \right),
\end{equation*}
and the expected number of labels requested among indices $1,\ldots,T$ is at most
\begin{multline*}
O\left( \dc_{\C}\left( \sqrt{\vc\change} \Log(1/(\vc\change)) \right) \vc \Log(\vc/\change) \Log(1/(\vc\change)) \left\lceil \frac{T}{M} \right\rceil \right)
\\ = O\left( \dc_{\C}\left( \sqrt{\vc\change} \Log(1/(\vc\change)) \right) \sqrt{\vc\change} \Log(\vc/\change) \Log(1/(\vc\change)) T \right).
\end{multline*}
\qed
\end{proof}